\documentclass[runningheads]{llncs}

\usepackage{eccv}

\usepackage{eccvabbrv}

\usepackage{graphicx}
\usepackage{booktabs}

\usepackage[accsupp]{axessibility}  

\usepackage{hyperref}

\usepackage{orcidlink}
\usepackage{amsmath}
\usepackage{amssymb}
\usepackage{mathtools}
\usepackage{tabularx}
\usepackage{booktabs}
\usepackage{algorithm}
\usepackage{multirow}
\usepackage{makecell}
\usepackage{xcolor,colortbl}
\usepackage{microtype}
\usepackage{algpseudocode}

\definecolor{head_purple}{rgb}{0.580, 0.404, 0.741}
\definecolor{head_red}{rgb}{0.882, 0.341, 0.349}
\definecolor{head_blue}{rgb}{0.30588, 0.4745, 0.6549}
\definecolor{Periwinkle}{rgb}{0.56,0.56,0.86}

\begin{document}

\title{Geometric Gradient Rectification for Safe Open-Set Semi-Supervised Learning} 

\titlerunning{Geometric Gradient Rectification for OSSL}

\author{Jiahe~Chen\inst{1}\textsuperscript{*} \and
Qian~Shao\inst{1}\textsuperscript{*} \and
Qiyuan~Chen\inst{1} \and
Jiaying~He\inst{1,3} \and
Jintai~Chen\inst{4} \and
Jian~Wu\inst{1,5,6}\textsuperscript{$\dagger$} \and
Hongxia~Xu\inst{1,2}\textsuperscript{$\dagger$}
}

\authorrunning{J.~Chen et al.}

\makeatletter
\newcommand{\inlineinst}{\stepcounter{@inst}\enspace$^{\the@inst}$\enspace}
\makeatother

\institute{College of Computer Science and Technology, Zhejiang University
\and Liangzhu Laboratory and WeDoctor Cloud
\and Shanghai Innovation Institute
\and The Hong Kong University of Science and Technology (Guangzhou)
\and State Key Laboratory of Transvascular Implantation Devices and TIDRI
\and Zhejiang Key Laboratory of Medical Imaging Artificial Intelligence
}

\makeatletter
\newcommand{\blfootnote}[1]{%
  \begingroup
  \renewcommand{\thefootnote}{}%
  \footnotetext{#1}%
  \addtocounter{footnote}{-1}%
  \endgroup
}
\makeatother

\maketitle

\blfootnote{\textsuperscript{*}Equal contribution.}
\blfootnote{\textsuperscript{$\dagger$}Corresponding authors.}

\begin{abstract}
Open-set semi-supervised learning aims to leverage unlabeled data that may contain out-of-distribution outliers while maintaining performance on in-distribution classes. Existing methods mainly follow two paradigms: filtering suspicious samples or incorporating unlabeled objectives with soft weighting. We argue that both face a common trade-off: aggressive filtering can discard informative but hard ID samples, whereas utilization can introduce auxiliary gradients that conflict with supervised learning when pseudo labels are wrong. We therefore shift the focus from sample selection to gradient-level control. We propose \textit{Geometric Gradient Rectification} (GGR), a plug-in framework that uses the supervised gradient as an anchor and projects conflicting auxiliary gradients onto an admissible region in gradient space. This makes the applied auxiliary update first-order non-opposing within the rectified coordinate block while preserving orthogonal components that may still carry useful representation signals. We further extend GGR with subspace-aware rectification to stabilize the anchor under noisy mini-batch gradients. Experiments on CIFAR and ImageNet benchmarks show that GGR improves representative OSSL baselines in most settings and yields gains in both closed-set generalization and open-set robustness. Code will be available at \url{https://github.com/JiaheChen2002/GGR}.
\keywords{Open-Set Semi-Supervised Learning \and Gradient Rectification \and Optimization Control}
\end{abstract}

\section{Introduction}

\begin{figure}[t] 
\centering \includegraphics[width=0.49\linewidth]{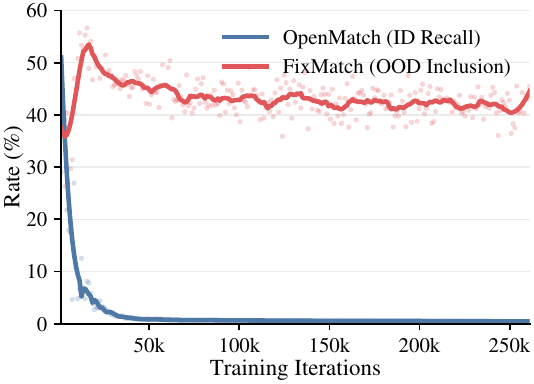} \hfill \includegraphics[width=0.49\linewidth]{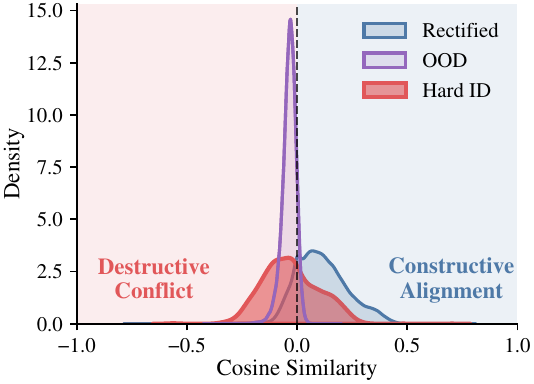} 
\caption{\textbf{Left: OSSL dilemma.} A trade-off between robustness and coverage that can lead to either feature starvation or noise domination.
\textbf{Right: Gradient geometry.} Density of cosine similarity between unlabeled gradients and the oracle supervised direction.
\textcolor{head_purple}{\textbf{OOD:}} mostly orthogonal noise;
\textcolor{head_red}{\textbf{Hard ID:}} misclassified samples induce adversarial gradients;
\textcolor{head_blue}{\textbf{Rectified:}} GGR projects conflicting components into an anchor-aligned subspace, shifting updates toward the constructive space.}
\label{fig:head} 
\vspace{-0.5cm}
\end{figure}

Semi-supervised learning (SSL) \cite{lee2013pseudo,meanteacher,oliver2018realistic,ssl_survey} effectively scales training by leveraging unlabeled data. However, standard SSL assumes a closed-set scenario, which rarely holds in real-world deployments where unlabeled streams contain Out-of-Distribution (OOD) outliers. In this open-set semi-supervised learning (OSSL) setting, naively enforcing consistency on OOD data can introduce substantial noise and degrade performance. Existing research mainly falls into two paradigms: \textit{filtering-based} and \textit{utilization-based} methods \cite{uasd,mtcf,openmatch,t2t,safestudent,prosub}. We argue that both paradigms face a recurring trade-off between rejecting harmful samples and preserving informative but ambiguous ones.

Filtering-based methods, such as OpenMatch \cite{openmatch}, prioritize robustness by rejecting suspicious samples. However, distinguishing between OOD noise and hard in-distribution (ID) samples is notoriously difficult in low-label regimes. To minimize false positives, these methods can become overly strict and discard valid hard ID samples, leading to feature starvation (Figure \ref{fig:head}, left). Conversely, utilization-based approaches like FixMatch \cite{fixmatch} and IOMatch \cite{iomatch} attempt to maximize data usage, often assigning outliers to a unified unknown class. This can misguide optimization: misclassifying a hard ID sample as unknown implicitly penalizes its true class probability and may weaken the learned ID manifold.

We argue that this dilemma arises because sample-level separation of OOD and hard ID data is intrinsically unreliable. Instead of hinging the pipeline on brittle categorical decisions, it is useful to examine the resulting optimization dynamics. Under softmax cross-entropy, assigning an unlabeled sample to an incorrect target can generate a gradient that opposes the oracle supervised update.

To make this concrete, we visualize the cosine similarity between unlabeled gradients and the oracle supervised direction in the right panel of Figure \ref{fig:head}. The distributions reveal two distinct behaviors:
\textbf{(i)} OOD outliers can be harmful when included in large quantities, as seen in FixMatch, by injecting variance that weakens the learning signal. Nevertheless, their gradients are often approximately orthogonal to the supervised direction, acting more like stochastic noise than systematic opposition.
\textbf{(ii)} Misclassified difficult IDs are often more destructive directionally. Because they share semantics with ID classes yet are pushed toward wrong targets, they can generate strong adversarial gradients that directly oppose the supervised objective.

These findings suggest a shift in perspective. Robustness in OSSL cannot rely solely on sample selection or category assignment. Instead, it can benefit from geometric control of optimization that suppresses harmful updates even when OOD/ID decisions are imperfect.

Motivated by this observation, we propose \textbf{G}eometric \textbf{G}radient \textbf{R}ectification (GGR), a framework that rectifies updates rather than discarding ambiguous samples or committing to potentially wrong targets. At its core, GGR performs post-hoc rectification in gradient space with respect to a supervised anchor. We further study two subspace-aware variants that summarize recent supervised descent directions in a low-dimensional basis to improve stability under noisy mini-batch gradients. In practice, this can reduce destructive interference from misguided difficult IDs, damp variance from bystander OODs, and preserve useful feature signals from ambiguous data.

Our main contributions are threefold:
\begin{itemize}
\item We identify the identification bottleneck in OSSL and reframe the dilemma of feature starvation versus active unlearning as a fundamental problem of gradient conflict.
\item We propose GGR, a plug-in optimization framework that enforces first-order non-opposition of the applied auxiliary update within the rectified coordinate block via gradient projection, without requiring perfect OOD detection.
\item Extensive experiments show that GGR improves strong OSSL baselines in most settings, particularly in low-label regimes.
\end{itemize}

\section{Related Work}

\subsection{Standard Semi-Supervised Learning} 
Mainstream semi-supervised learning (SSL) is largely built upon two core mechanisms: consistency regularization and pseudo-labeling. Early works \cite{pseudo-ensembles, temporal-ensembling, meanteacher, vat} improved generalization by enforcing prediction invariance under stochastic or adversarial perturbations, laying the foundation for modern SSL. Holistic methods like MixMatch \cite{mixmatch} and ReMixMatch \cite{remixmatch} synthesized these principles with distribution alignment and augmentation anchoring to stabilize training. Building on these advancements, FixMatch \cite{fixmatch} simplified the pipeline by unifying weak-to-strong consistency with confidence-based pseudo-labeling, establishing a robust baseline. Recent works further refine this framework from two perspectives: (1) Dynamic Thresholding: strategies like FlexMatch \cite{flexmatch} and FreeMatch \cite{freematch} introduce adaptive thresholds to mitigate the learning difficulty variance across classes; (2) Structure Awareness: methods such as CoMatch \cite{comatch} and SimMatch \cite{simmatch} incorporate graph-based or contrastive objectives to enhance feature discriminability.


\subsection{Open-Set Semi-Supervised Learning}
OSSL addresses the challenge where the unlabeled pool contains unknown classes \cite{ssb, ren2024partial, yang2024robust,yang2021s2osc}. Existing strategies primarily fall into two paradigms based on how they handle potential outliers: filtering and utilization.

\textbf{Filtering-based.} The dominant strategy prioritizes robustness by explicitly detecting and rejecting outliers. Early works like OpenMatch~\cite{openmatch} and UASD~\cite{uasd} rely on confidence thresholds or entropy to filter OOD samples.
To improve detection robustness in low-label regimes, recent state-of-the-art methods introduce more sophisticated selection mechanisms. ProSub \cite{prosub} advances this paradigm by modeling the feature space geometry, using angles to an ID subspace to derive probabilistic outlier scores rather than relying on brittle hard thresholds; SCOMatch~\cite{scomatch} maintains an OOD memory queue to calibrate confidence for better selection;
and UAGreg~\cite{uagreg} introduces graph regularization to smooth uncertainty estimates for soft weighting; DAC~\cite{dac} leverages prediction disagreement among diverse heads to identify and filter unstable samples. Despite these advances, the paradigm still reduces the problem largely to sample selection. This creates a recurring trade-off: stricter filtering can reduce OOD noise but also suppress informative difficult ID samples, ultimately causing feature starvation.

\textbf{Utilization-based.} To circumvent data waste, another line of research attempts to utilize all unlabeled data by assigning outliers to a unified target. IOMatch~\cite{iomatch} treats outliers as an additional class via a unified open-set target. Similarly, MTCF~\cite{mtcf} and T2T~\cite{t2t} leverage auxiliary detection heads or self-supervised signals to harvest representation benefits from OOD samples. Although retaining difficult samples alleviates data scarcity, it also introduces a risk of misguidance. Forcing diverse difficult ID samples into a single unknown category or auxiliary targets can generate gradients that conflict with the primary supervised objective and may weaken learned ID features.

\subsection{Optimization Control and Gradient Surgery}
Different from sample-level selection, optimization-based methods focus on resolving task interference directly in the gradient space. Techniques like PCGrad \cite{pcgrad} project conflicting gradients in multi-task learning to prevent negative transfer. However, standard gradient surgery assumes symmetric task importance. In OSSL, the requirement is asymmetric: the unsupervised objective should not oppose the supervised learning signal. Our GGR framework bridges this gap by keeping the supervised gradient fixed and rectifying only the auxiliary one; Appendix~\ref{app:pcgrad} provides a direct comparison to symmetric PCGrad.
Instead of relying on brittle sample selection or potentially destructive joint training, we use geometric gradient rectification to control auxiliary updates directly in gradient space while preserving the supervised anchor.

\section{Methodology}
\label{sec:method}

\subsection{Problem Formulation and Gradient Conflict}
\label{sec:formulation_conflict}

We consider OSSL with a labeled set $\mathcal{D}_l=\{(x_i^l,y_i^l)\}_{i=1}^{N_l}$ over $K$ ID classes and an unlabeled set $\mathcal{D}_u=\{x_i^u\}_{i=1}^{N_u}$ contaminated by OOD samples. A base algorithm optimizes a composite objective:
\begin{equation}
\min_{\theta}\; \mathcal{J}(\theta) := \mathcal{L}_s(\theta;\mathcal{D}_l) + \lambda_u\,\mathcal{L}_u(\theta;\mathcal{D}_u),
\label{eq:base_objective}
\end{equation}
where $\mathcal{L}_s$ represents the supervised loss on labeled data and $\mathcal{L}_u$ denotes an auxiliary objective. At iteration $t$, let $g_s := \nabla_\theta \mathcal{L}_s(\theta_t)$ and $g_u := \nabla_\theta \mathcal{L}_u(\theta_t)$ be the stochastic gradients computed on the current labeled and unlabeled batches, respectively. We denote the Euclidean inner product as $\langle \cdot,\cdot\rangle$.

A primary failure in this formulation stems not merely from the presence of OOD samples, but from how $\mathcal{L}_u$ processes ambiguous data. Assigning a difficult ID sample to an incorrect pseudo label or an unknown category can produce an unlabeled gradient $g_u$ that opposes the supervised descent direction $g_s$, potentially yielding $\langle g_s, g_u\rangle < 0$. Such destructive interference can weaken discriminative ID features by enforcing erroneous constraints.

This observation establishes a fundamental first-order non-interference requirement: the auxiliary update applied within the rectified scope should not oppose supervised progress at the current step. Formally, we seek a rectified auxiliary direction $\tilde{g}_u$ such that $\langle g_s, \tilde{g}_u\rangle \ge 0$.

\subsection{Geometric Gradient Rectification (GGR)}
\label{sec:ggr}

GGR is a framework comprising three rectifiers that share the same first-order non-opposition principle: a default \emph{vector-level rectifier} (VLR) derived in this subsection, and two subspace-aware variants---the \emph{orthogonal subspace rectifier} (OSR) and the \emph{conic subspace rectifier} (CSR)---presented in Section~\ref{sec:subspace}. We begin with VLR, which uses the instantaneous supervised gradient as the anchor.

Our goal is to ensure that the auxiliary update does not oppose the supervised descent direction at the current iteration. 
Let $g_s := \nabla_\theta \mathcal{L}_{\mathrm{s}}$ denote the supervised gradient and let $g_u := \nabla_\theta \mathcal{L}_{\mathrm{u}}$ denote the raw auxiliary gradient inside the selected surgery scope. To enforce the first-order non-interference constraint, we require the corrected auxiliary direction to satisfy
\begin{equation}
\langle \tilde g_u, g_s\rangle \ge 0.
\label{eq:safety_constraint_repeat}
\end{equation}
Geometrically, this constraint defines a closed half-space induced by $g_s$:
\begin{equation}
\mathcal{H}_{\mathrm{safe}}(g_s)
:= \left\{ d\in\mathbb{R}^D ~\middle|~ \langle d, g_s\rangle \ge 0 \right\}.
\end{equation}
We then obtain $\tilde g_u$ by the Euclidean projection of $g_u$ onto $\mathcal{H}_{\mathrm{safe}}(g_s)$:
\begin{equation}
\tilde{g}_u
~=~\arg\min_{d\in\mathbb{R}^D}\;\frac{1}{2}\|d-g_u\|_2^2 
\quad\text{s.t.}\quad \langle d, g_s\rangle \ge 0.
\label{eq:qp_halfspace}
\end{equation}
This formulation is intentionally asymmetric: the supervised gradient $g_s$ acts as an anchor that specifies the admissible region, while only the auxiliary gradient is modified. The same asymmetric anchoring principle carries over to the subspace variants in Section~\ref{sec:subspace}, where the anchor is replaced by a supervised subspace $U$ rather than the instantaneous $g_s$; in all three GGR rectifiers the supervised direction itself is never altered. Unlike symmetric gradient surgery methods such as PCGrad \cite{pcgrad}, GGR never alters the supervised anchor; a direct comparison is provided in Appendix~\ref{app:pcgrad}.

The projection in \eqref{eq:qp_halfspace} admits a closed-form solution, which we refer to as the \emph{vector-level rectifier} (VLR):
\begin{equation}
\tilde{g}_u =
\begin{cases}
g_u - \dfrac{\langle g_u, g_s\rangle}{\|g_s\|_2^2}\,g_s, 
& \text{if }\langle g_u,g_s\rangle<0\text{ and }\|g_s\|_2>0,\\[6pt]
g_u, & \text{otherwise.}
\end{cases}
\label{eq:ggr_vector}
\end{equation}
When $\langle g_u,g_s\rangle<0$, \eqref{eq:ggr_vector} removes exactly the component of $g_u$ that points toward $-g_s$, while preserving the orthogonal component $g_u - \frac{\langle g_u,g_s\rangle}{\|g_s\|_2^2}g_s$ that may still carry useful representation-learning signals. 
Algorithm~\ref{alg:ggr} implements VLR together with the two subspace variants of Section~\ref{sec:subspace}, applying full rectification whenever a conflict is detected; we do not introduce any additional soft gating coefficient.

\paragraph{Plug-and-play update.}
Inside the selected surgery scope, GGR modifies a baseline OSSL method by substituting the raw auxiliary gradient $g_u$ with its rectified counterpart $\tilde g_u$ before weighting and aggregation:
\begin{equation}
\theta_{t+1} \leftarrow \theta_t - \eta\left(g_s + \lambda_u \tilde g_u\right),
\label{eq:update}
\end{equation}
When rectification is applied only to a subset $\mathcal{P}\subseteq\theta$, Eq.~\eqref{eq:update} should be interpreted on the flattened coordinates of $\mathcal{P}$; parameters outside $\mathcal{P}$ follow the base update unchanged, as in Algorithm~\ref{alg:ggr}. Equivalently, because VLR is positively homogeneous, one may view GGR as rectifying the weighted auxiliary gradient $\lambda_u g_u$ instead; applying $\lambda_u$ before or after rectification is the same.
Importantly, GGR is purely a post-hoc operation in gradient space: it does not require explicit OOD detection, sample rejection, or any modification to the forward pass or the loss definitions of the underlying OSSL framework.

\subsection{Rectification via Subspaces for Robust Anchoring}
\label{sec:subspace}

In practice, the minibatch supervised gradient $g_s$ can be noisy, which may lead to unstable rectification decisions.
To obtain a more robust anchor, we maintain a compact orthonormal basis $U=[u_1,\dots,u_k]\in\mathbb{R}^{D\times k}$ (with $k\le d$) that summarizes recent supervised directions via periodic orthonormalization.
This defines a supervised subspace $\mathcal{S}=\mathrm{span}(U)$, which we use to rectify $g_u$ at the subspace level.
We consider two efficient variants:

\textbf{Orthogonal subspace rectification. (OSR)}
This conservative option removes the entire component of $g_u$ that lies in $\mathcal{S}$:
\begin{equation}
\tilde{g}_u^{\mathrm{orth}} = (I-UU^\top)\,g_u.
\label{eq:ggr_orth}
\end{equation}
Since $(I-UU^\top)$ projects onto the orthogonal complement of $\mathcal{S}$, we have $\langle d,\tilde g_u^{\mathrm{orth}}\rangle=0$ for all $d\in\mathcal{S}$. 
In particular, if the current supervised gradient lies in $\mathcal{S}$, then $\langle g_s,\tilde g_u^{\mathrm{orth}}\rangle=0$, so the auxiliary update is exactly orthogonal to the supervised anchor and therefore non-opposing.

\textbf{Conic Subspace Rectification (CSR).}
This less restrictive variant retains orthogonal signals while removing negative coordinates inside $\mathcal{S}$:
\begin{equation}
\tilde{g}_u^{\mathrm{signed}} 
= g_u - U\,\min(U^\top g_u, 0),
\label{eq:ggr_signed}
\end{equation}
where $\min(\cdot,0)$ is applied element-wise.
When $U$ is orthonormal, \eqref{eq:ggr_signed} is exactly the Euclidean projection of $g_u$ onto the convex cone
$\mathcal{C}(U)=\{d~|~\langle d,u_i\rangle\ge 0,\ \forall i\}$.
This operation removes only the directions that conflict with the stored supervised anchors, while preserving components that remain compatible.
Furthermore, whenever $g_s$ lies exactly in the conic hull of $\{u_i\}$, the rectified gradient satisfies $\langle g_s,\tilde g_u^{\mathrm{signed}}\rangle \ge 0$. When $g_s$ is only approximately represented by this cone, CSR should be interpreted as an approximate anchoring heuristic rather than a strict guarantee.

\paragraph{Summary.}
Algorithm~\ref{alg:ggr} presents the complete training procedure.
At each iteration, we compute $g_s$ from labeled data and $g_u$ from unlabeled objectives, optionally update the anchor subspace $U$ from recent supervised gradients, and rectify $g_u$ using either the basic VLR \eqref{eq:ggr_vector} or a subspace-level variant OSR \eqref{eq:ggr_orth} or CSR \eqref{eq:ggr_signed}.
The rectified auxiliary gradient is then combined with $g_s$ for the optimizer step, while the pseudo-labeling mechanism and the objective design of the base OSSL method remain unchanged.

\begin{algorithm}[!t]
\caption{Geometric Gradient Rectification as an Optimization Plug-in}
\label{alg:ggr}
\begin{algorithmic}[1]
\Require Labeled minibatch $\mathcal{B}_{\ell}$, unlabeled minibatch $\mathcal{B}_{u}$, model parameters $\theta$, \Require Supervised loss $\mathcal{L}_{\mathrm{s}}$, auxiliary loss $\mathcal{L}_{\mathrm{u}}$, unsupervised weight $\lambda_u$
\Require Surgery scope $\mathcal{P}\subseteq\theta$ (both / backbone / head), Subspace dimension $d \ge 0$
\Require Rectification mode $m\in\{\texttt{VLR}, \texttt{CSR}, \texttt{OSR}\}$
\State Initialize orthonormal basis $U \leftarrow \emptyset$ \Comment{$U\in\mathbb{R}^{D\times k},\,k\le d$}

\For{$t=1,2,\dots,\mathrm{MaxIter}$}
    \State Compute $\mathcal{L}_{\mathrm{s}}(\mathcal{B}_{\ell};\theta)$ and $\mathcal{L}_{\mathrm{u}}(\mathcal{B}_{u};\theta)$ \Comment{\textit{Forward Pass}}
    \State $g_s \leftarrow \nabla_{\mathcal{P}} \mathcal{L}_{\mathrm{s}}$,\hspace{6pt} $g_u \leftarrow \nabla_{\mathcal{P}} \mathcal{L}_{\mathrm{u}}$ \Comment{\textit{Backward Pass}}
    \State $\mathbf{g}_s \leftarrow \mathrm{Flatten}(g_s)$,\hspace{6pt} $\mathbf{g}_u \leftarrow \mathrm{Flatten}(g_u)$
    
    \If{$d > 0$}
        \State $U \leftarrow \textsc{UpdateSubspace}(U, \mathbf{g}_s, d)$ \Comment{\textit{Subspace Maintenance}}
    \EndIf
    
    \State $\mathbf{g}_u \leftarrow \textsc{Rectify}(\mathbf{g}_u, \mathbf{g}_s, U, m)$ \Comment{\textit{Geometric Rectification}}
    \State $g_u \leftarrow \mathrm{Unflatten}(\mathbf{g}_u)$
    \State $\nabla_{\mathcal{P}} \leftarrow g_s + \lambda_u g_u$,\hspace{6pt} $\nabla_{\theta \setminus \mathcal{P}} \leftarrow \nabla_{\theta \setminus \mathcal{P}} \bigl(\mathcal{L}_{\mathrm{s}} + \lambda_u \mathcal{L}_{\mathrm{u}}\bigr)$ \Comment{\textit{Apply Gradients}}
    \State Update $\theta$ with the optimizer step (e.g., SGD/Adam)
\EndFor

\vspace{1mm}
\Function{\textsc{UpdateSubspace}}{$U, \mathbf{g}_s, d$}
    \State $\mathbf{v} \leftarrow \mathbf{g}_s - U(U^\top \mathbf{g}_s)$
    \If{$\|\mathbf{v}\|_2 > 0$}
        \State $U \leftarrow [U,\, \mathbf{v} / \|\mathbf{v}\|_2]$ \Comment{\textit{Gram-Schmidt append}}
    \EndIf
    \State \Return $\mathrm{Orthonormalize}(\,U_{[:,\, -d:]}\,)$ \Comment{\textit{Keep last $d$ columns \& apply QR}}
\EndFunction

\vspace{1mm}
\Function{\textsc{Rectify}}{$\mathbf{g}_u, \mathbf{g}_s, U, m$}
    \If{$m = \texttt{VLR}$ \textbf{and} $\langle \mathbf{g}_u, \mathbf{g}_s \rangle < 0$ \textbf{and} $\|\mathbf{g}_s\|_2 > 0$}
        \State \Return $\mathbf{g}_u - \frac{\langle \mathbf{g}_u, \mathbf{g}_s \rangle}{\|\mathbf{g}_s\|_2^2}\mathbf{g}_s$ \Comment{\textit{Ref: Eq. \eqref{eq:ggr_vector}}}
    \ElsIf{$m = \texttt{OSR}$ \textbf{and} $U \neq \emptyset$}
        \State \Return $\mathbf{g}_u - U(U^\top \mathbf{g}_u)$ \Comment{\textit{Ref: Eq. \eqref{eq:ggr_orth}}}
    \ElsIf{$m = \texttt{CSR}$ \textbf{and} $U \neq \emptyset$}
        \State \Return $\mathbf{g}_u - U\min(U^\top \mathbf{g}_u, \mathbf{0})$ \Comment{\textit{Ref: Eq. \eqref{eq:ggr_signed}}}
    \EndIf
    \State \Return $\mathbf{g}_u$ \Comment{\textit{Fallback (no conflict or empty $U$)}}
\EndFunction
\end{algorithmic}
\end{algorithm}

\section{Theoretical Analysis}
\label{sec:theory}

We begin by establishing the geometric optimality of GGR. Given a supervised gradient $g_s$ and a raw auxiliary gradient $g_u$, the rectified update $\tilde{g}_u$ represents the unique minimizer for our objective. Among all possible directions that satisfy the constraint $\langle d, g_s\rangle \ge 0$, it applies the minimum possible modification to $g_u$ measured in the Euclidean norm. This leads directly to a crucial identity regarding asymmetric alignment. The inner product $\langle g_s, \tilde{g}_u\rangle$ exactly equals $\max(0, \langle g_s, g_u\rangle)$, which is inherently greater than or equal to zero.

This alignment yields a scoped first-order non-interference statement on the selected surgery coordinates. Let $z\in\mathbb{R}^D$ denote the flattened parameter vector inside the chosen surgery scope $\mathcal{P}$, while the coordinates outside $\mathcal{P}$ are held fixed for this local analysis. We assume that the resulting supervised-loss slice is smooth with a Lipschitz constant $L$. Under this condition, defining the block update direction at iteration $t$ as $v_t := g_s + \lambda_u \tilde{g}_u$, where $g_s=\nabla_z \mathcal{L}_s(z_t;\theta_{t,\setminus\mathcal{P}})$ and $\tilde g_u$ denotes the rectified raw auxiliary gradient on the same block, yields $\langle g_s, v_t\rangle \ge \|g_s\|_2^2$. Consequently, the partial update $z_{t+1}=z_t-\eta v_t$ yields one-step descent of the supervised objective with respect to this coordinate block under a suitable step size. If the learning rate $\eta$ is bounded by $\dfrac{\langle g_s,v_t\rangle}{L\|v_t\|_2^2}$, the loss slice decreases according to the following relation:
\begin{equation}
\mathcal{L}_s(z_{t+1};\theta_{t,\setminus\mathcal{P}}) \le \mathcal{L}_s(z_t;\theta_{t,\setminus\mathcal{P}}) - \frac{\eta}{2}\langle g_s,v_t\rangle.
\label{eq:monotone}
\end{equation}
Specifically, for any non-trivial scope gradient $g_s$, this descent is strict for that block step. We further formalize the optimization degradation through the cumulative \emph{applied} conflict regret over $T$ iterations, defined on the weighted auxiliary update actually used by the optimizer, namely $g_{\mathrm{app}}^{(t)}:=\lambda_u\tilde g_u^{(t)}$. For VLR, this applied regret is exactly zero, whereas raw pre-rectification conflicts may still occur before surgery.

To formalize the dilemma between safety and starvation, we introduce a stylized mixture model characterizing the unlabeled gradient as a combination of aligned signals from difficult IDs and conflicting components from outliers. Let $g_s^\star$ denote the non-zero population supervised gradient with a normalized direction $\hat{g}_s^\star := g_s^\star/\|g_s^\star\|_2$. We formulate the unlabeled gradient as $g_u = \beta g_s^\star + \epsilon_u$ for a scalar alignment coefficient $\beta$. For theoretical tractability, we assume the noise term $\epsilon_u$ is zero-mean and almost surely orthogonal to $g_s^\star$. This assumption is consistent with the empirical density in the right panel of Figure \ref{fig:head}, where OOD gradients behave predominantly as orthogonal stochastic noise. Any persistent negative alignment can be absorbed into $\beta$ as a negative contribution, so the single-step projection analysis still captures the destructive component that VLR removes.

Under this model, we evaluate the expected progress along the supervised direction $\hat{g}_s^\star$ for three distinct strategies. A strategy of hard rejection, representing filtering and feature starvation, yields an expected progress equal to $\|g_s^\star\|_2$. A naive utilization strategy, representing the risk of poisoning, yields an expected progress of $(1+\lambda_u \mathbb{E}[\beta])\|g_s^\star\|_2$. Finally, our GGR approach yields an expected progress of $(1+\lambda_u \mathbb{E}[\max(0,\beta)])\|g_s^\star\|_2$.

This formulation reveals a strict inequality when $\lambda_u>0$ and the probability of $\beta>0$ is strictly greater than zero. Specifically, when the expected value of $\beta$ is negative, naive utilization performs strictly worse than filtering. Within this stylized model, GGR matches or exceeds the filtering baseline in expectation, while preserving positive aligned components and removing destructive ones. Appendix~\ref{app:pcgrad} further contrasts this asymmetric anchor with symmetric PCGrad.


\section{Experiments}

\subsection{Experimental Setup}

\textbf{Benchmarks and Datasets.}
We evaluate our method on standard image classification SSL benchmarks under the open-set setting, following prior OSSL works \cite{iomatch,dac,openmatch,uagreg,scomatch}. 
Specifically, we construct OSSL tasks on CIFAR-10/100 \cite{cifar} and ImageNet \cite{imagenet}, where the labeled set contains $K$ ID classes while the unlabeled pool is a mixture of ID and OOD samples from unknown classes.
We adopt the class-splitting protocol from Kong \etal~\cite{dac} to define seen/unseen class partitions, and evaluate multiple settings by varying (i) the number of seen/unseen classes and (ii) the amount of labeled data per seen class.

\begin{table*}[tbp]
  \centering
  \caption{Closed-set classification accuracy on CIFAR-10/100 with varying seen/unseen class splits and labeled set sizes.}
  \scalebox{0.65}{
  \setlength{\tabcolsep}{0.5mm}{
    \begin{tabular}{c|c ccc ccc ccc}
    \toprule
    \multicolumn{2}{c}{Dataset} 
      & \multicolumn{3}{c}{CIFAR-10} 
      & \multicolumn{6}{c}{CIFAR-100} \\
    \cmidrule(r){1-2} \cmidrule(r){3-5} \cmidrule(r){6-11}

    \multicolumn{2}{c}{Class split} 
      & \multicolumn{3}{c}{6 / 4} 
      & \multicolumn{3}{c}{20 / 80} 
      & \multicolumn{3}{c}{50 / 50} \\
    \cmidrule(r){1-2} \cmidrule(r){3-5} \cmidrule(r){6-8} \cmidrule{9-11}

    \multicolumn{2}{c}{No. of labeled samples}
      & 5 & 10 & 25
      & 5 & 10 & 25
      & 5 & 10 & 25 \\
    \midrule
    
\multicolumn{1}{c}{}
& Fully Supervised
& 26.72\scalebox{0.8}{$\pm$1.74} & 
  29.30\scalebox{0.8}{$\pm$1.30} & 
  38.33\scalebox{0.8}{$\pm$3.41} & 
  22.72\scalebox{0.8}{$\pm$1.86} & 
  29.00\scalebox{0.8}{$\pm$2.02} & 
  43.18\scalebox{0.8}{$\pm$1.26} &
  18.16\scalebox{0.8}{$\pm$0.29} & 
  26.01\scalebox{0.8}{$\pm$0.52} & 
  44.05\scalebox{0.8}{$\pm$1.87} \\ 

\midrule
\multicolumn{1}{c}{} & MixMatch (NIPS'19)
& 37.95\scalebox{0.8}{$\pm$3.96} & 
  44.27\scalebox{0.8}{$\pm$2.27} & 
  58.24\scalebox{0.8}{$\pm$0.81} & 
  32.47\scalebox{0.8}{$\pm$1.91} & 
  44.48\scalebox{0.8}{$\pm$1.87} & 
  55.82\scalebox{0.8}{$\pm$1.15} & 
  30.45\scalebox{0.8}{$\pm$0.23} & 
  41.99\scalebox{0.8}{$\pm$0.51} & 
  53.59\scalebox{0.8}{$\pm$1.07} \\ 

\multicolumn{1}{c}{}
& FixMatch (NIPS'20)
& 83.65\scalebox{0.8}{$\pm$5.03} 
& 90.74\scalebox{0.8}{$\pm$1.93} 
& 93.47\scalebox{0.8}{$\pm$0.08}
& 47.68\scalebox{0.8}{$\pm$3.21} 
& 55.87\scalebox{0.8}{$\pm$1.13} 
& 66.75\scalebox{0.8}{$\pm$0.92}
& 52.31\scalebox{0.8}{$\pm$2.95} 
& 62.38\scalebox{0.8}{$\pm$0.24} 
& 69.03\scalebox{0.8}{$\pm$0.69} \\

\multicolumn{1}{c}{}
& CoMatch (ICCV'21)
& 84.83\scalebox{0.8}{$\pm$4.11} 
& 89.86\scalebox{0.8}{$\pm$2.37} 
& 92.78\scalebox{0.8}{$\pm$0.47}
& 51.20\scalebox{0.8}{$\pm$2.90} 
& 61.05\scalebox{0.8}{$\pm$1.81} 
& 67.47\scalebox{0.8}{$\pm$0.86}
& 46.79\scalebox{0.8}{$\pm$0.78} 
& 59.81\scalebox{0.8}{$\pm$0.50} 
& 69.94\scalebox{0.8}{$\pm$0.42} \\

\multicolumn{1}{c}{}
& SimMatch (CVPR'22)
& 79.30\scalebox{0.8}{$\pm$4.99} 
& 86.31\scalebox{0.8}{$\pm$3.42} 
& 90.74\scalebox{0.8}{$\pm$1.75}
& 49.58\scalebox{0.8}{$\pm$0.09} 
& 56.72\scalebox{0.8}{$\pm$0.75} 
& 65.97\scalebox{0.8}{$\pm$0.33}
& 55.09\scalebox{0.8}{$\pm$1.42} 
& 65.15\scalebox{0.8}{$\pm$0.48} 
& 70.23\scalebox{0.8}{$\pm$0.40} \\

\multicolumn{1}{c}{}
& SoftMatch (ICLR'23)
& 80.51\scalebox{0.8}{$\pm$5.21} 
& 83.18\scalebox{0.8}{$\pm$1.69} 
& 86.45\scalebox{0.8}{$\pm$1.49}
& 53.17\scalebox{0.8}{$\pm$2.05} 
& 58.88\scalebox{0.8}{$\pm$1.16} 
& 66.82\scalebox{0.8}{$\pm$0.52}
& 55.56\scalebox{0.8}{$\pm$0.28} 
& 63.39\scalebox{0.8}{$\pm$0.97} 
& 69.05\scalebox{0.8}{$\pm$0.51} \\

\multicolumn{1}{c}{} & MTCF (ECCV'20)
& 40.66\scalebox{0.8}{$\pm$5.07} 
& 45.62\scalebox{0.8}{$\pm$1.75} 
& 56.04\scalebox{0.8}{$\pm$2.13}
& 33.95\scalebox{0.8}{$\pm$2.21} 
& 46.20\scalebox{0.8}{$\pm$1.59} 
& 57.37\scalebox{0.8}{$\pm$1.47}
& 30.73\scalebox{0.8}{$\pm$0.38} 
& 43.29\scalebox{0.8}{$\pm$0.22} 
& 54.98\scalebox{0.8}{$\pm$0.90} \\

\multicolumn{1}{c}{}
& Safe-Stu (CVPR'22)
& 50.16\scalebox{0.8}{$\pm$6.43} 
& 61.21\scalebox{0.8}{$\pm$9.17} 
& 82.27\scalebox{0.8}{$\pm$0.77}
& 43.70\scalebox{0.8}{$\pm$3.80} 
& 53.90\scalebox{0.8}{$\pm$1.23} 
& 61.52\scalebox{0.8}{$\pm$0.92}
& 42.97\scalebox{0.8}{$\pm$0.90} 
& 55.83\scalebox{0.8}{$\pm$0.84} 
& 66.17\scalebox{0.8}{$\pm$0.65} \\

\multicolumn{1}{c}{}
& UAGreg (AAAI'24)
& 85.78\scalebox{0.8}{$\pm$2.27} 
& 87.11\scalebox{0.8}{$\pm$2.12} 
& 91.48\scalebox{0.8}{$\pm$0.37}
& 57.92\scalebox{0.8}{$\pm$0.76} 
& 64.02\scalebox{0.8}{$\pm$1.39} 
& 68.77\scalebox{0.8}{$\pm$1.11}
& 60.71\scalebox{0.8}{$\pm$0.99} 
& 66.19\scalebox{0.8}{$\pm$0.33} 
& 71.20\scalebox{0.8}{$\pm$0.28} \\

\midrule

\multicolumn{1}{c}{}
& OpenMatch (NIPS'21)
& 47.26\scalebox{0.8}{$\pm$1.20} 
& 51.21\scalebox{0.8}{$\pm$1.82} 
& 64.73\scalebox{0.8}{$\pm$1.78}
& 43.72\scalebox{0.8}{$\pm$1.24} 
& 55.78\scalebox{0.8}{$\pm$0.02} 
& 63.78\scalebox{0.8}{$\pm$1.01}
& 44.55\scalebox{0.8}{$\pm$1.56} 
& 56.19\scalebox{0.8}{$\pm$1.53} 
& 65.65\scalebox{0.8}{$\pm$0.52} \\

\rowcolor{blue!10}
\multicolumn{1}{c}{}
& \textbf{+ Ours}
& \textbf{47.38\scalebox{0.8}{$\pm$1.37}} 
& \textbf{52.75\scalebox{0.8}{$\pm$2.65}} 
& \textbf{65.23\scalebox{0.8}{$\pm$1.88}}
& \textbf{43.73\scalebox{0.8}{$\pm$0.70}} 
& \textbf{55.82\scalebox{0.8}{$\pm$0.59}} 
& \textbf{64.42\scalebox{0.8}{$\pm$0.90}} 
& \textbf{45.11\scalebox{0.8}{$\pm$0.73}}
& \textbf{56.45\scalebox{0.8}{$\pm$0.62}}  
& \textbf{66.36\scalebox{0.8}{$\pm$0.92}} \\

\midrule

\multicolumn{1}{c}{}
& IOMatch (ICCV'23)
& 89.87\scalebox{0.8}{$\pm$1.26} 
& 91.16\scalebox{0.8}{$\pm$1.47} 
& 93.61\scalebox{0.8}{$\pm$0.29}
& 56.48\scalebox{0.8}{$\pm$1.86} 
& 63.00\scalebox{0.8}{$\pm$1.51} 
& 67.32\scalebox{0.8}{$\pm$0.89}
& 59.49\scalebox{0.8}{$\pm$0.99} 
& 65.33\scalebox{0.8}{$\pm$1.11} 
& 70.43\scalebox{0.8}{$\pm$1.29} \\

\rowcolor{blue!10}
\multicolumn{1}{c}{}
& \textbf{+ Ours}
& \textbf{91.87\scalebox{0.8}{$\pm$1.44}} 
& \textbf{92.62\scalebox{0.8}{$\pm$0.93}} 
& \textbf{93.64\scalebox{0.8}{$\pm$0.49}}
& \textbf{56.75\scalebox{0.8}{$\pm$1.45}} 
& \textbf{63.18\scalebox{0.8}{$\pm$1.21}} 
& \textbf{67.63\scalebox{0.8}{$\pm$0.30}}
& \textbf{60.61\scalebox{0.8}{$\pm$0.89}} 
& \textbf{65.94\scalebox{0.8}{$\pm$0.47}} 
& \textbf{70.89\scalebox{0.8}{$\pm$0.33}} \\

\midrule

\multicolumn{1}{c}{}
& DAC (TNNLS'25)
& 86.69\scalebox{0.8}{$\pm$3.74} 
& 91.66\scalebox{0.8}{$\pm$0.90} 
& 93.02\scalebox{0.8}{$\pm$1.23}
& 53.30\scalebox{0.8}{$\pm$1.45} 
& 58.42\scalebox{0.8}{$\pm$0.56} 
& 67.72\scalebox{0.8}{$\pm$0.65}
& 56.94\scalebox{0.8}{$\pm$0.77} 
& 64.84\scalebox{0.8}{$\pm$1.10} 
& 70.60\scalebox{0.8}{$\pm$0.44} \\

\rowcolor{blue!10}
\multicolumn{1}{c}{}
& \textbf{+ Ours}
& \textbf{87.63\scalebox{0.8}{$\pm$2.71}} 
& \textbf{92.56\scalebox{0.8}{$\pm$1.02}} 
& \textbf{93.25\scalebox{0.8}{$\pm$0.79}}
& \textbf{53.45\scalebox{0.8}{$\pm$3.40}} 
& \textbf{60.10\scalebox{0.8}{$\pm$2.64}} 
& \textbf{68.03\scalebox{0.8}{$\pm$0.76}} 
& \textbf{57.19\scalebox{0.8}{$\pm$0.25}}
& \textbf{65.45\scalebox{0.8}{$\pm$0.20}}  
& \textbf{70.68\scalebox{0.8}{$\pm$0.26}} \\
\bottomrule
\end{tabular}%
}}
\label{tab:cifar-closedset}
\end{table*}

\textbf{Baselines.}
We compare our approach against a comprehensive suite of state-of-the-art methods, categorized into two groups:
(1) Standard SSL methods, including MixMatch \cite{mixmatch}, FixMatch \cite{fixmatch}, CoMatch \cite{comatch}, SimMatch \cite{simmatch}, and SoftMatch \cite{softmatch}. Following prior work \cite{iomatch,dac}, we exclude earlier deep SSL methods that often underperform labeled-only baselines under OOD contamination.
(2) Open-Set SSL methods, including MTCF \cite{mtcf}, OpenMatch \cite{openmatch}, Safe-Student \cite{safestudent}, IOMatch \cite{iomatch}, UAGreg \cite{uagreg}, DAC \cite{dac}.

\textbf{Evaluation Metrics.}
For closed-set evaluation, our primary metric is closed-set accuracy on the test set restricted to seen classes.
This measures the core goal of OSSL: leveraging an unlabeled pool contaminated with outliers to improve ID classification.
For open-set evaluation, we additionally evaluate open-set classification on the test set that contains both seen and unseen classes.
Following prior OSSL works \cite{iomatch,dac}, we treat all unseen classes as a single unknown category.
Because the open-set test distribution is often class-imbalanced, we adopt Balanced Accuracy (BA) \cite{balanced_accuracy} as the open-set metric.
For each method, open-set BA is computed using the checkpoint selected by the best closed-set performance, ensuring consistent model selection across methods.

\textbf{Fairness of comparisons.}
To ensure rigorous fairness and reproducibility, all methods are implemented and evaluated within the USB codebase \cite{usb}. We maintain strict consistency across all comparisons by using identical data splits, augmentation policies, and common hyperparameters. Regarding backbone architectures, we employ WideResNet-28-2 \cite{wideresnet} for CIFAR benchmarks and ResNet-18 \cite{resnet} for ImageNet \cite{usb}. Baseline-specific hyperparameters are initialized from original papers and tuned on the validation split, allowing us to successfully reproduce results consistent with the original reports \cite{iomatch,dac}. Finally, to avoid unfair comparisons due to varying convergence speeds, we report the performance of the best checkpoint selected based on the closed-set validation accuracy.

\begin{table*}[tbp]
    
  \centering
  \caption{Open-set classification accuracy on CIFAR-10/100 with varying seen/unseen class splits and labeled set sizes.}
  \scalebox{0.65}{
  \setlength{\tabcolsep}{0.5mm}{
    \begin{tabular}{c|c ccc ccc ccc}
    \toprule
    \multicolumn{2}{c}{Dataset} 
      & \multicolumn{3}{c}{CIFAR-10} 
      & \multicolumn{6}{c}{CIFAR-100} \\
    \cmidrule(r){1-2} \cmidrule(r){3-5} \cmidrule(r){6-11}

    \multicolumn{2}{c}{Class split} 
      & \multicolumn{3}{c}{6 / 4} 
      & \multicolumn{3}{c}{20 / 80} 
      & \multicolumn{3}{c}{50 / 50} \\
    \cmidrule(r){1-2} \cmidrule(r){3-5} \cmidrule(r){6-8} \cmidrule{9-11}

    \multicolumn{2}{c}{No. of labeled samples}
      & 5 & 10 & 25
      & 5 & 10 & 25
      & 5 & 10 & 25 \\
    \midrule

\multicolumn{1}{c}{}
& IOMatch (ICCV'23)
& 72.50\scalebox{0.8}{$\pm$1.05} & 74.11\scalebox{0.8}{$\pm$1.20} & 76.83\scalebox{0.8}{$\pm$0.38}
& 45.71\scalebox{0.8}{$\pm$1.11} & 51.59\scalebox{0.8}{$\pm$0.35} & 56.71\scalebox{0.8}{$\pm$0.42}
& 48.74\scalebox{0.8}{$\pm$1.50} & 54.02\scalebox{0.8}{$\pm$0.93} & 59.80\scalebox{0.8}{$\pm$1.02} \\

\rowcolor{blue!10}
\multicolumn{1}{c}{}
& \textbf{+ Ours}
& \textbf{74.13\scalebox{0.8}{$\pm$1.64}} & \textbf{75.58\scalebox{0.8}{$\pm$0.53}} & \textbf{76.86\scalebox{0.8}{$\pm$0.43}}
& \textbf{45.73\scalebox{0.8}{$\pm$0.67}} & \textbf{51.68\scalebox{0.8}{$\pm$0.21}} & \textbf{56.93\scalebox{0.8}{$\pm$0.35}}
& \textbf{49.93\scalebox{0.8}{$\pm$0.95}} & \textbf{54.87\scalebox{0.8}{$\pm$0.11}} & \textbf{60.18\scalebox{0.8}{$\pm$0.33}} \\

\midrule

\multicolumn{1}{c}{}
& DAC (TNNLS'25)
& 72.02\scalebox{0.8}{$\pm$3.07} & 76.26\scalebox{0.8}{$\pm$0.82} & 77.48\scalebox{0.8}{$\pm$0.85}
& 49.41\scalebox{0.8}{$\pm$0.54} & 55.18\scalebox{0.8}{$\pm$0.36} & 64.34\scalebox{0.8}{$\pm$0.45}
& 54.65\scalebox{0.8}{$\pm$0.08} & 62.79\scalebox{0.8}{$\pm$1.51} & 68.10\scalebox{0.8}{$\pm$0.10} \\

\rowcolor{blue!10}
\multicolumn{1}{c}{}
& \textbf{+ Ours}
& \textbf{73.72\scalebox{0.8}{$\pm$2.07}} & \textbf{77.19\scalebox{0.8}{$\pm$0.69}} & \textbf{78.66\scalebox{0.8}{$\pm$1.61}}
& 48.94\scalebox{0.8}{$\pm$3.81} & \textbf{57.01\scalebox{0.8}{$\pm$2.94}} & \textbf{64.55\scalebox{0.8}{$\pm$0.80}}
& 53.77\scalebox{0.8}{$\pm$0.31} & \textbf{63.15\scalebox{0.8}{$\pm$0.17}} & \textbf{68.15\scalebox{0.8}{$\pm$1.18}} \\
\bottomrule
\end{tabular}%
}}
\label{tab:cifar-openset}
\vspace{-0.2cm}
\end{table*}

\textbf{Implementation details.}
For our GGR framework, we adopt a unified hyperparameter configuration across all main experiments: we use VLR as the default rectification mode. To balance computational efficiency and representational stability, the gradient surgery scope ($\mathcal{P}$) is applied to the encoder backbone, meaning task-specific heads are updated via standard gradients. All method-specific hyperparameters of the base OSSL algorithms are kept identical to their baseline settings so that the effect of gradient rectification can be isolated as clearly as possible. All reported results are the mean and standard deviation over 3 independent runs with different random seeds.

\subsection{Main Results}
\label{sec:main_results}

We evaluate GGR on CIFAR-10/100 and ImageNet-30 under standard open-set splits and label budgets, reporting (i) closed-set accuracy on the seen-class test set and (ii) open-set BA on the mixed test set.

\textbf{CIFAR-10/100.}
Tables~\ref{tab:cifar-closedset} and~\ref{tab:cifar-openset} summarize the closed-set and open-set performance on CIFAR-10/100.
As a plug-in module, GGR improves closed-set and open-set accuracy in most settings when combined with representative filtering-based and utilization/disagreement-based methods, including OpenMatch, IOMatch, and DAC. The gains are generally larger in more challenging regimes with scarce labels and large unseen partitions. For example, relative to the corresponding baselines, GGR improves closed-set accuracy by up to +2.00 points and open-set accuracy by up to +1.83 points in the low-label settings. These results are consistent with our hypothesis that controlling destructive gradient interference can benefit OSSL, especially when pseudo-label noise and ID/OOD ambiguity are severe. For brevity, we denote the CIFAR-10 tasks with 6 labeled classes and 5, 10, or 25 labeled samples per class as CIFAR-6-30, CIFAR-6-60, and CIFAR-6-150, respectively.

\textbf{ImageNet-30.}
Table \ref{tab:in30} reports results on ImageNet-30 with the 20/10 split under 1\% and 5\% labeled ratios.
Despite the increased scale and visual diversity, GGR improves strong OSSL baselines across both filtering and utilization settings on ImageNet-30.
The gains are larger under the 1\% labeled regime, where pseudo-label noise and ID/OOD ambiguity are more severe.
These results suggest that optimization-level rectification is complementary to existing detection heuristics and can remain useful beyond small benchmarks.

Taken together, these results indicate that rectifying auxiliary gradients in a geometric manner can improve OSSL while enhancing both ID generalization and open-set recognition without relying solely on hard sample rejection.

\begin{table}[tb]
  \centering
  \caption{Closed-set and open-set accuracy on ImageNet-30 with the class split of $20/10$.}
  \scalebox{0.92}{
  \setlength{\tabcolsep}{1.15mm}{
    \begin{tabular}{cccccc}
    \toprule
    \multicolumn{2}{c}{Evaluation} & \multicolumn{2}{c}{Closed-Set} & \multicolumn{2}{c}{Open-Set} \\
    \cmidrule(r){1-2} \cmidrule(r){3-4} \cmidrule(r){5-6}
    \multicolumn{2}{c}{Labeled ratio} & 1\% & 5\% & 1\% & 5\% \\
    \midrule
    \multicolumn{2}{c}{Fully Supervised}   & 31.87\scalebox{0.8}{$\pm$0.97} & 56.08\scalebox{0.8}{$\pm$0.88} & -- & -- \\
    \midrule
    \multicolumn{2}{c}{MixMatch (NIPS'19)} & 51.42\scalebox{0.8}{$\pm$1.27} & 76.97\scalebox{0.8}{$\pm$0.65} & -- & -- \\
    \multicolumn{2}{c}{FixMatch (NIPS'20)} & 73.92\scalebox{0.8}{$\pm$1.74} & 88.82\scalebox{0.8}{$\pm$0.88} & -- & -- \\
    \multicolumn{2}{c}{CoMatch (ICCV'21)}  & 83.72\scalebox{0.8}{$\pm$0.84} & 91.17\scalebox{0.8}{$\pm$0.48} & -- & -- \\
    \multicolumn{2}{c}{SimMatch (CVPR'22)} & 82.78\scalebox{0.8}{$\pm$1.59} & 91.88\scalebox{0.8}{$\pm$0.48} & -- & -- \\
    \multicolumn{2}{c}{SoftMatch (ICLR'23)} & 81.18\scalebox{0.8}{$\pm$3.19} & 90.00\scalebox{0.8}{$\pm$0.23} & -- & -- \\
    \multicolumn{2}{c}{MTCF (ECCV'20)} & 49.33\scalebox{0.8}{$\pm$1.50} & 71.62\scalebox{0.8}{$\pm$0.46} & -- & -- \\
    \multicolumn{2}{c}{Safe-Student (CVPR'22)} & 71.47\scalebox{0.8}{$\pm$2.91} & 88.42\scalebox{0.8}{$\pm$0.06} & -- & -- \\
    \multicolumn{2}{c}{UAGreg (AAAI'24)} & 80.28\scalebox{0.8}{$\pm$0.35} & 91.02\scalebox{0.8}{$\pm$0.34} & -- & -- \\
    
    \midrule
    
    \multicolumn{2}{c}{OpenMatch (NIPS'21)} & 58.55\scalebox{0.8}{$\pm$2.10} & 86.72\scalebox{0.8}{$\pm$0.91} & 14.75\scalebox{0.8}{$\pm$1.18} & 61.51\scalebox{0.8}{$\pm$3.63} \\
    
    \rowcolor{blue!10}
    \multicolumn{2}{c}{+ Ours} & \textbf{59.15\scalebox{0.8}{$\pm$1.00}} & \textbf{87.43\scalebox{0.8}{$\pm$0.50}} & \textbf{15.56\scalebox{0.8}{$\pm$2.38}} & \textbf{68.15\scalebox{0.8}{$\pm$1.25}} \\
    
    \midrule
    
    \multicolumn{2}{c}{IOMatch (ICCV'23)} & 85.18\scalebox{0.8}{$\pm$1.80} & 90.32\scalebox{0.8}{$\pm$0.52} & 74.88\scalebox{0.8}{$\pm$1.56} & 81.03\scalebox{0.8}{$\pm$1.01} \\
    
    \rowcolor{blue!10}
    \multicolumn{2}{c}{+ Ours} & \textbf{85.90\scalebox{0.8}{$\pm$1.07}} & \textbf{90.82\scalebox{0.8}{$\pm$0.58}} & \textbf{75.46\scalebox{0.8}{$\pm$0.60}} & \textbf{82.14\scalebox{0.8}{$\pm$1.28}} \\
    
    \midrule
    
    \multicolumn{2}{c}{DAC (TNNLS'25)} & 86.40\scalebox{0.8}{$\pm$0.08} & 93.25\scalebox{0.8}{$\pm$0.20} & 77.78\scalebox{0.8}{$\pm$2.10} & 87.81\scalebox{0.8}{$\pm$1.65} \\
    
    \rowcolor{blue!10}
    \multicolumn{2}{c}{+ Ours} & \textbf{87.40\scalebox{0.8}{$\pm$0.15}} & \textbf{93.33\scalebox{0.8}{$\pm$0.29}} & \textbf{78.40\scalebox{0.8}{$\pm$3.62}} & \textbf{88.74\scalebox{0.8}{$\pm$1.27}} \\
    
    \bottomrule
    \end{tabular}}
  }
  \label{tab:in30}
  \vspace{-0.2cm}
\end{table}

\subsection{Ablation Studies and Further Analysis} \label{sec:ablation}

\textbf{Rectification via subspaces.} While VLR serves as a hyperparameter-free and efficient default rectifier at the vector level, it anchors updates dynamically to the current batch gradient $g_{s}$. In regimes with limited labels, this single-step estimate exhibits high variance, leaving the rectification susceptible to inaccurate corrections. To address this, we explore subspace-aware rectification, which summarizes recent supervised gradients into a low-dimensional basis to provide temporal smoothing. We compare two variants evaluated across subspace dimensions $d \in \{1, 5, 10, 20, 30, 40, 50\}$ on the CIFAR-100 dataset with a 50/50 seen/unseen split under two label budgets.

As Figure~\ref{fig:subspace} shows, the two operators exhibit a conservativeness-utilization trade-off. OSR is the more conservative option: it projects the auxiliary gradient onto the orthogonal complement of the supervised subspace, completely eliminating any component within it. Empirically, this leads to more stable improvements in open-set balanced accuracy across a wide range of dimensions. Conversely, CSR performs a cone projection, clipping only negative coordinates while retaining positively aligned components. By preserving more auxiliary signals, it can achieve higher peak closed-set accuracy (e.g., 58.47 at $d=10$). However, it is also more sensitive to the quality of the estimated basis; misaligned positive components can still induce destructive interference, resulting in higher variance and non-monotonic performance across different dimensions.

Selecting a moderate subspace dimension $d$ is important for balancing conservativeness and utilization. Small dimensions (e.g., $d \in \{1, 5\}$) provide only a limited anchor and may fail to sufficiently prevent gradient conflicts under extreme label scarcity. Conversely, excessively large dimensions capture too many supervised descent directions and can suppress beneficial unsupervised signals. This over-constraining effect is particularly noticeable in easier regimes (e.g., CIFAR100-50-1250), where auxiliary guidance is already relatively reliable. Overall, a moderate $d$ appears to capture the primary descent directions without unnecessarily discarding useful representation signals.

\begin{table}[tbp]
  \centering
  \caption{Ablation study on the rectification scope. Comparison of performance when GGR is applied to the encoder backbone, task-specific heads, or both modules.}
  \label{tab:ablation}
  
  \begin{minipage}[t]{0.48\textwidth}
    \centering
    \resizebox{\linewidth}{!}{
      \begin{tabular}{lcc}
        \toprule
        \multicolumn{3}{c}{\textbf{CIFAR10-6-30}} \\
        \cmidrule{1-3}
        Method & Closed-Set & Open-Set \\
        \midrule
        DAC               & 86.69\scalebox{0.8}{$\pm$3.74} & 72.02\scalebox{0.8}{$\pm$3.07} \\
        + Ours (backbone) & 87.63\scalebox{0.8}{$\pm$2.71} & 73.72\scalebox{0.8}{$\pm$2.07} \\
        + Ours (head)     & \textbf{90.38}\scalebox{0.8}{$\pm$0.43} & \textbf{75.24}\scalebox{0.8}{$\pm$0.35} \\
        + Ours (both)     & \underline{88.47}\scalebox{0.8}{$\pm$3.02} & \underline{75.06}\scalebox{0.8}{$\pm$3.55} \\
        \midrule
        IOMatch           & 89.87\scalebox{0.8}{$\pm$1.26} & 72.50\scalebox{0.8}{$\pm$1.05} \\
        + Ours (backbone) & \textbf{91.87}\scalebox{0.8}{$\pm$1.44} & \textbf{74.13}\scalebox{0.8}{$\pm$1.64} \\
        + Ours (head)     & \underline{90.34}\scalebox{0.8}{$\pm$0.49} & \underline{72.61}\scalebox{0.8}{$\pm$0.33} \\
        + Ours (both)     & 89.56\scalebox{0.8}{$\pm$0.84} & 71.99\scalebox{0.8}{$\pm$0.88} \\
        \bottomrule
      \end{tabular}
    }
  \end{minipage}
  \hfill
  \begin{minipage}[t]{0.48\textwidth}
    \centering
    \resizebox{\linewidth}{!}{
      \begin{tabular}{lcc}
        \toprule
        \multicolumn{3}{c}{\textbf{CIFAR100-50-250}} \\
        \cmidrule{1-3}
        Method & Closed-Set & Open-Set \\
        \midrule
        DAC               & 56.94\scalebox{0.8}{$\pm$0.77} & 54.65\scalebox{0.8}{$\pm$0.07} \\
        + Ours (backbone) & 57.19\scalebox{0.8}{$\pm$0.25} & 53.76\scalebox{0.8}{$\pm$0.31} \\
        + Ours (head)     & \underline{58.04}\scalebox{0.8}{$\pm$0.65} & \underline{55.15}\scalebox{0.8}{$\pm$0.84} \\
        + Ours (both)     & \textbf{58.50}\scalebox{0.8}{$\pm$1.91} & \textbf{55.46}\scalebox{0.8}{$\pm$2.43} \\
        \midrule
        IOMatch           & 59.49\scalebox{0.8}{$\pm$0.99} & 48.74\scalebox{0.8}{$\pm$1.50} \\
        + Ours (backbone) & \textbf{60.61}\scalebox{0.8}{$\pm$0.89} & \textbf{49.93}\scalebox{0.8}{$\pm$0.95} \\
        + Ours (head)     & \underline{60.38}\scalebox{0.8}{$\pm$0.50} & \underline{49.69}\scalebox{0.8}{$\pm$0.60} \\
        + Ours (both)     & 59.56\scalebox{0.8}{$\pm$0.68} & 48.91\scalebox{0.8}{$\pm$0.43} \\
        \bottomrule
      \end{tabular}
    }
  \end{minipage}
  
\end{table}

\begin{figure}[htbp]
    \centering
    \includegraphics[width=0.99\linewidth]{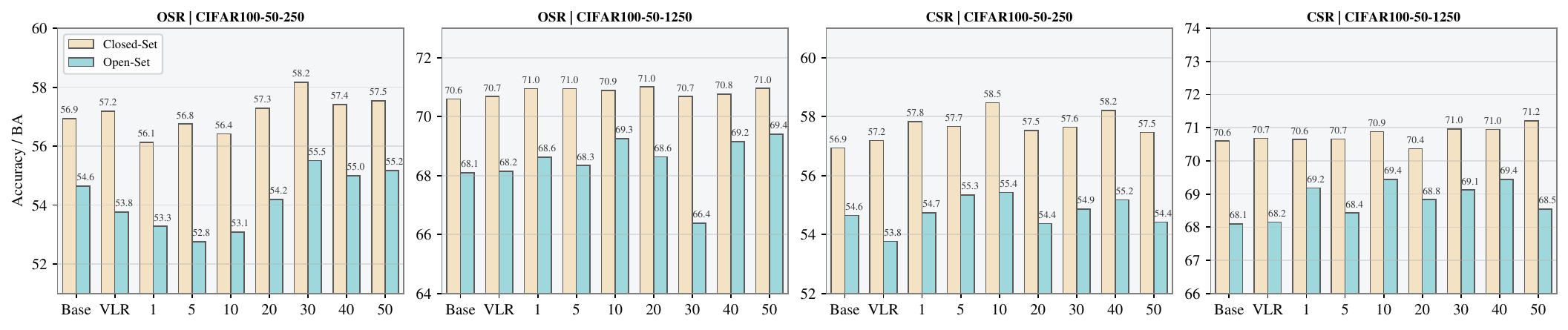}
    \caption{Effect of subspace-aware rectification on CIFAR-100  across varying subspace dimensions $d$. We compare OSR and CSR projection variants under two label budgets.}
    \label{fig:subspace}
\end{figure}

In summary, VLR is an efficient default, while subspace anchoring can provide more stable behavior and temporal smoothing. The orthogonal variant offers steadier open-set improvements across dimensions, whereas the signed variant can yield slightly higher peak performance at the cost of greater sensitivity to hyperparameter tuning.

\textbf{Gradient-conflict diagnostics.} To empirically validate our optimization-level non-interference principle, we track gradient alignment during training. We define a \emph{conflict event} as occurring when the auxiliary gradient ($g_u^{(t)}$) directly opposes the supervised direction ($g_s^{(t)}$), satisfying $\langle g_s^{(t)}, g_u^{(t)}\rangle < 0$. Figure~\ref{fig:conflict_regret} analyzes this hazard across three metrics:
\textbf{First}, the raw conflict rate (left) tracks the moving-average frequency of conflicts before rectification. The baseline's non-trivial rate indicates that auxiliary objectives frequently generate destructive interference. Because GGR operates post-hoc in gradient space, it does not remove this underlying loss-level hazard; instead, it rectifies the update before the parameter step.
\textbf{Second}, the residual conflict rate (middle) evaluates the actual applied update $g_{\mathrm{app}}^{(t)}$. While the baseline suffers from unmitigated conflicts, GGR drives this rate close to zero. This is consistent with our first-order non-interference analysis on the rectified coordinate block: the applied auxiliary update no longer pushes against supervised progress there.
\textbf{Finally}, the cumulative conflict regret (right), measured as $\sum_{t}\max(0,-\langle g_s^{(t)}, g_{\mathrm{app}}^{(t)}\rangle)$, captures the accumulated optimization damage over time. The baseline accumulates substantially more regret from counterproductive updates, whereas GGR remains nearly flat. Together, these diagnostics are consistent with the view that geometric rectification suppresses harmful update directions before they are applied.

\begin{figure}[tbp]
    \centering
    \includegraphics[width=0.99\linewidth]{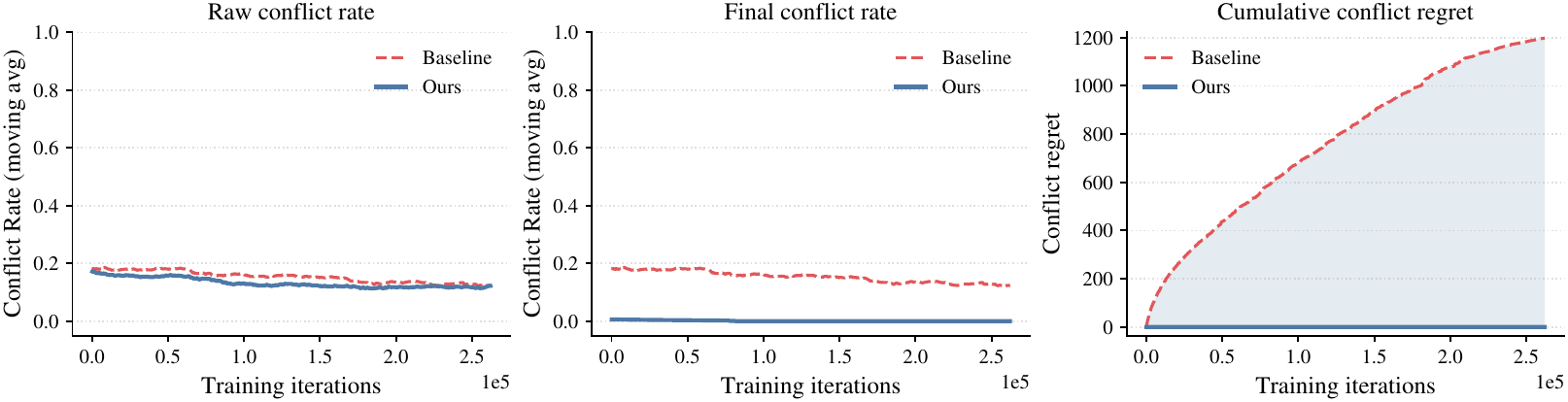}
    \caption{\textbf{Conflict diagnostics during training.} Tracking the occurrence and impact of gradient conflicts between the baseline and our GGR framework.}
    \label{fig:conflict_regret}
\end{figure}

\textbf{Surgery Scope.}
We investigate where to apply GGR by restricting the rectification to
(i) the encoder backbone, (ii) the non-backbone modules (all parameters outside the backbone; denoted as head for brevity), or (iii) both (all parameters).
Table \ref{tab:ablation} reports results on CIFAR10-6-30 and CIFAR100-50-250 for two representative OSSL paradigms.
For DAC, applying GGR to the head yields the largest gain on CIFAR10-6-30 in both closed-set accuracy and open-set BA,
while rectifying both backbone and head performs best on the more challenging CIFAR100-50-250 setting.
This suggests that, under DAC-style objectives with multiple auxiliary heads, gradient conflicts are not confined to the final decision boundary:
mitigation at auxiliary branches is crucial, and full-model rectification becomes increasingly beneficial as the class mismatch grows.
For IOMatch, rectifying the backbone already provides consistent improvements over the baseline across both datasets,
indicating that conflict-free updates can directly benefit representation learning in utilization-based OSSL.
Overall, these results support that the optimal rectification scope can depend on the base algorithm and task difficulty,
and applying GGR to either the backbone or the entire model offers a robust default choice.

\textbf{Efficiency and cost analysis.}
Table~\ref{tab:efficiency} reports the wall-clock time and memory cost of GGR on top of DAC. VLR adds only a modest memory increase and about $11\%$ time overhead over the baseline. With a subspace dimension of $d=10$, OSR keeps the overhead within $\sim$$1\%$ and CSR within $\sim$$11\%$, while adding only a few hundred megabytes of memory. Increasing $d$ to $50$ further raises the cost without proportional accuracy gains, so VLR and the $d=10$ subspace variants serve as the practical defaults, balancing efficiency with the accuracy benefits analyzed above.

\begin{table}[tbp]
\centering
\caption{Computational cost and efficiency analysis of GGR. The relative time indicates the overhead compared to the baseline.}
\label{tab:efficiency}
\resizebox{0.95\linewidth}{!}{
\begin{tabular}{lccccc}
\toprule
\textbf{Method} & \textbf{Dim} & \textbf{Time(ms)}&  \textbf{Memory(MB)} & \textbf{Relative Time}  & \textbf{Accuracy(\%)}\\
\midrule
DAC    & --      & 220 & 9018 & 1.00& 56.06 \\
\midrule
+ GGR (VLR) & -- & 244 & 9336 & 1.11 &  56.90 \\
+ GGR (OSR) & 10    & 223 & 9336 & 1.01 &  56.60\\
+ GGR (OSR)  & 50   & 293 & 9572 &  1.33 & 57.06\\
+ GGR (CSR)  & 10   & 245 & 9336 & 1.11 & \underline{57.26}\\
+ GGR (CSR) & 50    & 316 & 9572 & 1.44 &  \textbf{57.38}\\
\bottomrule
\end{tabular}
}
\end{table}


\section{Limitations}
\label{app:limitations}

The preceding results are local and first-order. They show that the auxiliary update does not oppose the supervised descent direction at the current iterate, but they do not imply global optimality or monotonic decrease of the full composite objective \(\mathcal{L}_s + \lambda_u \mathcal{L}_u\). In addition, when the mini-batch supervised gradient is itself highly noisy or uninformative, the vector-level anchor may become conservative or unstable; the subspace variants partially mitigate this issue via temporal smoothing. Finally, GGR is complementary to, rather than a replacement for, better open-set modeling and pseudo-labeling. Stronger OOD detection or uncertainty estimation can further improve the usefulness of the rectified auxiliary signal.

\section{Conclusion}
In this paper, we study OSSL from the perspective of optimization geometry. We argue that errors in ID/OOD decisions can lead either to feature starvation or to destructive gradient interference. To address this issue, we propose GGR, a plug-in framework that shifts the focus from sample-level rejection to optimization-level control. By using the supervised gradient as an anchor, GGR projects conflicting auxiliary gradients onto an admissible region and makes the applied auxiliary update first-order non-opposing within the rectified coordinate block while preserving orthogonal signals that may remain useful. Theoretical analysis and empirical results show that GGR improves strong OSSL baselines in most settings, with larger gains in challenging low-label regimes. Overall, these results suggest that geometric control of the optimization trajectory is complementary to existing detection heuristics in open-set learning.

\section*{Acknowledgements}
This research was partially supported by Fundamental and Interdisciplinary Disciplines Breakthrough Plan of the Ministry of Education of China No.~JYB\allowbreak 2025\allowbreak XDXM601, ``Pioneer'' and ``Leading Goose'' R\&D Program of Zhejiang under Grant No.~2025\allowbreak C02120, and State Key Laboratory of Transvascular Implantation Devices under Grant No.~SKLTID\allowbreak 2024003.

\bibliographystyle{splncs04}
\bibliography{main}

\newpage
\appendix
\renewcommand{\theHsection}{appendix.\Alph{section}}
\renewcommand{\theHsubsection}{appendix.\Alph{section}.\arabic{subsection}}

\section{Proofs and Additional Theoretical Analysis}
\label{app:theory}

For clarity, throughout the appendix \(g_u\) denotes the raw auxiliary gradient, \(\tilde g_u\) denotes its rectified counterpart before weighting, and \(g_{\mathrm{app}}:=\lambda_u \tilde g_u\) denotes the weighted auxiliary update actually applied by the optimizer. For the default vector-level rectifier (VLR), we write
\[
\tilde g_u = \Pi_{\mathcal{H}_{\mathrm{safe}}(g_s)}(g_u),
\qquad
\mathcal{H}_{\mathrm{safe}}(g_s)
=
\{d \in \mathbb{R}^D : \langle d, g_s \rangle \ge 0\}.
\]

\subsection{Half-space projection and alignment identity}

\begin{proposition}[Closed-form projection onto the safe half-space]
\label{prop:halfspace}
For any \(g_s, g_u \in \mathbb{R}^D\), the Euclidean projection of \(g_u\) onto \(\mathcal{H}_{\mathrm{safe}}(g_s)\) is uniquely given by
\[
\Pi_{\mathcal{H}_{\mathrm{safe}}(g_s)}(g_u)
=
\begin{cases}
g_u - \dfrac{\langle g_u, g_s\rangle}{\|g_s\|_2^2} g_s,
& \text{if } \langle g_u, g_s\rangle < 0 \text{ and } \|g_s\|_2 > 0,\\[6pt]
g_u,
& \text{otherwise}.
\end{cases}
\]
\end{proposition}

\begin{proof}
If \(g_s = 0\), the constraint is vacuous and the claim is immediate. Assume \(g_s \neq 0\). Consider
\[
\min_{d \in \mathbb{R}^D} \frac{1}{2}\|d-g_u\|_2^2
\quad
\text{s.t.}
\quad
\langle d, g_s\rangle \ge 0.
\]
The objective is strongly convex and the feasible set is closed and convex, hence the minimizer is unique. The Lagrangian is
\[
\mathcal{L}(d,\mu)
=
\frac{1}{2}\|d-g_u\|_2^2 - \mu \langle d, g_s\rangle,
\qquad \mu \ge 0.
\]
Stationarity gives
\[
\nabla_d \mathcal{L}(d,\mu)=0
\quad\Longrightarrow\quad
d^\star = g_u + \mu g_s.
\]

If \(\langle g_u, g_s\rangle \ge 0\), then \(g_u\) is already feasible, so \(\mu=0\) and \(d^\star=g_u\).

If \(\langle g_u, g_s\rangle < 0\), the constraint must be active at optimum. By complementary slackness,
\[
\langle d^\star, g_s\rangle = 0,
\]
which implies
\[
\langle g_u, g_s\rangle + \mu \|g_s\|_2^2 = 0
\quad\Longrightarrow\quad
\mu = -\frac{\langle g_u, g_s\rangle}{\|g_s\|_2^2} > 0.
\]
Substituting this value into \(d^\star = g_u + \mu g_s\) yields the claimed form.
\end{proof}

\begin{proposition}[Alignment identity and minimal intervention]
\label{prop:alignment}
Let \(\tilde g_u = \Pi_{\mathcal{H}_{\mathrm{safe}}(g_s)}(g_u)\). Then
\[
\langle g_s, \tilde g_u\rangle
=
\max\!\bigl(0,\langle g_s, g_u\rangle\bigr)
\ge 0.
\]
Moreover, if \(\langle g_s, g_u\rangle < 0\) and \(g_s \neq 0\), then
\[
g_u - \tilde g_u
=
\frac{\langle g_u,g_s\rangle}{\|g_s\|_2^2} g_s,
\qquad
\|g_u - \tilde g_u\|_2
=
\frac{|\langle g_u,g_s\rangle|}{\|g_s\|_2},
\]
and
\[
\|\tilde g_u\|_2^2
=
\|g_u\|_2^2
-
\frac{\langle g_u,g_s\rangle^2}{\|g_s\|_2^2}.
\]
\end{proposition}

\begin{proof}
If \(g_s=0\), then \(\langle g_s,\tilde g_u\rangle=0\), so the statement is trivial. Assume \(g_s \neq 0\).

If \(\langle g_s, g_u\rangle \ge 0\), then Proposition~\ref{prop:halfspace} gives \(\tilde g_u = g_u\), hence
\[
\langle g_s,\tilde g_u\rangle
=
\langle g_s,g_u\rangle
=
\max(0,\langle g_s,g_u\rangle).
\]

If \(\langle g_s, g_u\rangle < 0\), Proposition~\ref{prop:halfspace} gives
\[
\tilde g_u
=
g_u - \frac{\langle g_u,g_s\rangle}{\|g_s\|_2^2} g_s.
\]
Taking the inner product with \(g_s\),
\[
\langle g_s,\tilde g_u\rangle
=
\langle g_s,g_u\rangle
-
\frac{\langle g_u,g_s\rangle}{\|g_s\|_2^2}\langle g_s,g_s\rangle
=
0
=
\max(0,\langle g_s,g_u\rangle).
\]
The remaining identities follow by direct substitution:
\[
g_u-\tilde g_u
=
\frac{\langle g_u,g_s\rangle}{\|g_s\|_2^2} g_s,
\]
and therefore
\[
\|g_u-\tilde g_u\|_2
=
\frac{|\langle g_u,g_s\rangle|}{\|g_s\|_2}.
\]
Finally,
\[
\|\tilde g_u\|_2^2
=
\left\|g_u - \frac{\langle g_u,g_s\rangle}{\|g_s\|_2^2}g_s\right\|_2^2
=
\|g_u\|_2^2
-
2\frac{\langle g_u,g_s\rangle^2}{\|g_s\|_2^2}
+
\frac{\langle g_u,g_s\rangle^2}{\|g_s\|_2^2}
=
\|g_u\|_2^2
-
\frac{\langle g_u,g_s\rangle^2}{\|g_s\|_2^2}.
\]
\end{proof}

\begin{remark}[Applying \(\lambda_u\) before or after rectification]
\label{rem:homogeneous}
For any \(\lambda_u \ge 0\),
\[
\Pi_{\mathcal{H}_{\mathrm{safe}}(g_s)}(\lambda_u g_u)
=
\lambda_u \Pi_{\mathcal{H}_{\mathrm{safe}}(g_s)}(g_u).
\]
Hence applying the unsupervised weight before rectification or multiplying the rectified auxiliary update afterward is equivalent for VLR.
\end{remark}

\subsection{First-order non-interference and one-step supervised descent}

\begin{proposition}[Scoped first-order non-interference and one-step descent on the surgery block]
\label{prop:safety_descent}
Assume that, with coordinates outside the selected surgery scope \(\mathcal{P}\) held fixed, the supervised loss is \(L\)-smooth as a function of the flattened scope coordinates \(z\in\mathbb{R}^D\). At iteration \(t\), let
\[
g_s^{(t)} = \nabla_z \mathcal{L}_s(z_t;\theta_{t,\setminus\mathcal{P}}),
\qquad
v_t = g_s^{(t)} + \lambda_u \tilde g_u^{(t)},
\qquad
z_{t+1} = z_t - \eta v_t,
\]
where \(\tilde g_u^{(t)}\) is the VLR-rectified raw auxiliary gradient on the same scope and \(\lambda_u \ge 0\). Then
\[
\langle g_s^{(t)}, v_t\rangle
\ge
\|g_s^{(t)}\|_2^2.
\]
Furthermore, if
\[
0 < \eta \le \frac{\langle g_s^{(t)}, v_t\rangle}{L\|v_t\|_2^2},
\]
then
\[
\mathcal{L}_s(z_{t+1};\theta_{t,\setminus\mathcal{P}})
\le
\mathcal{L}_s(z_t;\theta_{t,\setminus\mathcal{P}})
-
\frac{\eta}{2}\langle g_s^{(t)}, v_t\rangle.
\]
In particular, the decrease is strict whenever \(g_s^{(t)} \neq 0\).
\end{proposition}

\begin{proof}
By Proposition~\ref{prop:alignment},
\[
\langle g_s^{(t)}, \tilde g_u^{(t)}\rangle \ge 0.
\]
Since \(\lambda_u \ge 0\),
\[
\langle g_s^{(t)}, \lambda_u \tilde g_u^{(t)}\rangle \ge 0.
\]
Therefore,
\[
\langle g_s^{(t)}, v_t\rangle
=
\langle g_s^{(t)}, g_s^{(t)}\rangle
+
\langle g_s^{(t)}, \lambda_u \tilde g_u^{(t)}\rangle
\ge
\|g_s^{(t)}\|_2^2.
\]

Since \(\mathcal{L}_s(\cdot;\theta_{t,\setminus\mathcal{P}})\) is \(L\)-smooth,
\[
\mathcal{L}_s(z_{t+1};\theta_{t,\setminus\mathcal{P}})
\le
\mathcal{L}_s(z_t;\theta_{t,\setminus\mathcal{P}})
+
\langle \nabla_z \mathcal{L}_s(z_t;\theta_{t,\setminus\mathcal{P}}), z_{t+1}-z_t\rangle
+
\frac{L}{2}\|z_{t+1}-z_t\|_2^2.
\]
Substituting \(\nabla_z \mathcal{L}_s(z_t;\theta_{t,\setminus\mathcal{P}})=g_s^{(t)}\) and \(z_{t+1}-z_t = -\eta v_t\), we obtain
\[
\mathcal{L}_s(z_{t+1};\theta_{t,\setminus\mathcal{P}})
\le
\mathcal{L}_s(z_t;\theta_{t,\setminus\mathcal{P}})
-
\eta \langle g_s^{(t)}, v_t\rangle
+
\frac{L\eta^2}{2}\|v_t\|_2^2.
\]
If
\[
\eta \le \frac{\langle g_s^{(t)}, v_t\rangle}{L\|v_t\|_2^2},
\]
then
\[
\frac{L\eta^2}{2}\|v_t\|_2^2
\le
\frac{\eta}{2}\langle g_s^{(t)}, v_t\rangle,
\]
and thus
\[
\mathcal{L}_s(z_{t+1};\theta_{t,\setminus\mathcal{P}})
\le
\mathcal{L}_s(z_t;\theta_{t,\setminus\mathcal{P}})
-
\frac{\eta}{2}\langle g_s^{(t)}, v_t\rangle.
\]
If \(g_s^{(t)}\neq 0\), the first part implies \(\langle g_s^{(t)},v_t\rangle \ge \|g_s^{(t)}\|_2^2 > 0\), so the decrease is strict.
\end{proof}

\begin{corollary}[Zero post-rectification conflict regret]
\label{cor:zero_regret}
Define the weighted applied auxiliary update \(g_{\mathrm{app}}^{(t)}:=\lambda_u \tilde g_u^{(t)}\) and the post-rectification \emph{applied} conflict regret
\[
\mathcal{R}^{\mathrm{app}}_T
=
\sum_{t=1}^{T}
\max\!\Bigl(0, -\langle g_s^{(t)}, g_{\mathrm{app}}^{(t)}\rangle \Bigr).
\]
Then \(\mathcal{R}^{\mathrm{app}}_T = 0\) for all \(T\).
\end{corollary}

\begin{proof}
By Proposition~\ref{prop:alignment}, \(\langle g_s^{(t)}, \tilde g_u^{(t)}\rangle \ge 0\) for every \(t\). Since \(\lambda_u \ge 0\), we have \(\langle g_s^{(t)}, g_{\mathrm{app}}^{(t)}\rangle \ge 0\). Hence each summand is zero.
\end{proof}

\begin{remark}[Raw conflict versus applied conflict]
The corollary above is intentionally defined on the auxiliary update actually executed by the optimizer. The raw pre-rectification gradient \(g_u^{(t)}\) may still conflict with \(g_s^{(t)}\), which is precisely the phenomenon captured by the raw conflict-rate diagnostic in Figure~\ref{fig:conflict_regret}.
\end{remark}

\subsection{Subspace-aware rectification}

\begin{proposition}[OSR: exact and approximate non-opposition]
\label{prop:osr}
Let \(U \in \mathbb{R}^{D\times k}\) have orthonormal columns and define
\[
\tilde g_u^{\mathrm{orth}} = (I-UU^\top) g_u.
\]
Then
\[
\langle g_s, \tilde g_u^{\mathrm{orth}} \rangle
=
\langle (I-UU^\top) g_s, g_u \rangle.
\]
Consequently, if \(g_s \in \mathrm{span}(U)\), then
\[
\langle g_s, \tilde g_u^{\mathrm{orth}} \rangle = 0.
\]
More generally,
\[
\bigl|
\langle g_s, \tilde g_u^{\mathrm{orth}} \rangle
\bigr|
\le
\mathrm{dist}\!\bigl(g_s,\mathrm{span}(U)\bigr)\,\|g_u\|_2.
\]
\end{proposition}

\begin{proof}
Since \(I-UU^\top\) is symmetric,
\[
\langle g_s, \tilde g_u^{\mathrm{orth}}\rangle
=
\langle g_s, (I-UU^\top)g_u\rangle
=
\langle (I-UU^\top)g_s, g_u\rangle.
\]
If \(g_s \in \mathrm{span}(U)\), then \((I-UU^\top)g_s = 0\), proving the exact non-opposition statement.

For the approximate bound, note that
\[
\|(I-UU^\top)g_s\|_2
=
\mathrm{dist}(g_s,\mathrm{span}(U)).
\]
Applying Cauchy--Schwarz yields
\[
\bigl|
\langle g_s, \tilde g_u^{\mathrm{orth}} \rangle
\bigr|
=
\bigl|
\langle (I-UU^\top)g_s, g_u\rangle
\bigr|
\le
\|(I-UU^\top)g_s\|_2\,\|g_u\|_2.
\]
\end{proof}

\begin{proposition}[CSR as a cone projection and its non-opposition condition]
\label{prop:csr}
Let \(U \in \mathbb{R}^{D\times k}\) have orthonormal columns, and define
\[
C(U) = \{d \in \mathbb{R}^D : U^\top d \succeq 0\}.
\]
Then
\[
\tilde g_u^{\mathrm{csr}}
=
g_u - U\min(U^\top g_u, 0)
\]
is exactly the Euclidean projection of \(g_u\) onto \(C(U)\).

Moreover, writing
\[
g_s = Uc + r,
\qquad
r \perp \mathrm{span}(U),
\]
we have
\[
\langle g_s, \tilde g_u^{\mathrm{csr}} \rangle
=
c^\top \max(U^\top g_u, 0)
+
\langle r, (I-UU^\top)g_u\rangle.
\]
In particular, if \(g_s \in \mathrm{cone}(U)\), i.e., \(r=0\) and \(c \succeq 0\), then
\[
\langle g_s, \tilde g_u^{\mathrm{csr}} \rangle \ge 0.
\]
\end{proposition}

\begin{proof}
Let
\[
a = U^\top g_u,
\qquad
b = (I-UU^\top)g_u,
\]
so that \(g_u = Ua + b\) with \(b \perp \mathrm{span}(U)\). Any vector \(d\in\mathbb{R}^D\) can be written as
\[
d = Uz + b + q,
\]
where \(z\in\mathbb{R}^k\) and \(q \perp \mathrm{span}(U)\). Then
\[
U^\top d = z,
\]
so the constraint \(U^\top d \succeq 0\) is equivalent to \(z \succeq 0\). The objective becomes
\[
\frac{1}{2}\|d-g_u\|_2^2
=
\frac{1}{2}\|U(z-a)+q\|_2^2
=
\frac{1}{2}\|z-a\|_2^2 + \frac{1}{2}\|q\|_2^2,
\]
where we used the orthonormality of \(U\). Hence the optimum is attained at
\[
q^\star = 0,
\qquad
z^\star = \max(a,0),
\]
which gives
\[
d^\star
=
U\max(a,0)+b
=
g_u - U\min(U^\top g_u,0).
\]

For the non-opposition identity, note that
\[
\tilde g_u^{\mathrm{csr}}
=
U\max(U^\top g_u,0) + (I-UU^\top)g_u.
\]
Therefore,
\[
\begin{aligned}
\langle g_s, \tilde g_u^{\mathrm{csr}} \rangle
&= \langle Uc+r,\; U\max(U^\top g_u,0) + (I-UU^\top)g_u \rangle \\
&= c^\top \max(U^\top g_u,0) + \langle r,(I-UU^\top)g_u\rangle.
\end{aligned}
\]
If \(r=0\) and \(c\succeq 0\), the first term is nonnegative and the second term vanishes, proving the claim.
\end{proof}

\subsection{Stylized mixture model}

\begin{proposition}[Expected progress under the stylized mixture model]
\label{prop:mixture}
Let \(g_s^\star \neq 0\) be the population supervised gradient, and let
\[
\hat g_s^\star = \frac{g_s^\star}{\|g_s^\star\|_2}.
\]
Assume the auxiliary gradient admits the decomposition
\[
g_u = \beta g_s^\star + \epsilon_u,
\]
where \(\beta\) is a scalar alignment coefficient and \(\epsilon_u\) satisfies
\[
\mathbb{E}[\epsilon_u]=0,
\qquad
\langle \epsilon_u, g_s^\star \rangle = 0
\quad \text{a.s.}
\]
Consider the following three update rules:
\[
v_{\mathrm{filter}} = g_s^\star,
\qquad
v_{\mathrm{naive}} = g_s^\star + \lambda_u g_u,
\qquad
v_{\mathrm{ggr}} = g_s^\star + \lambda_u \Pi_{\mathcal{H}_{\mathrm{safe}}(g_s^\star)}(g_u).
\]
Then
\[
\mathbb{E}\!\left[\langle v_{\mathrm{filter}}, \hat g_s^\star\rangle\right]
=
\|g_s^\star\|_2,
\]
\[
\mathbb{E}\!\left[\langle v_{\mathrm{naive}}, \hat g_s^\star\rangle\right]
=
\bigl(1+\lambda_u \mathbb{E}[\beta]\bigr)\|g_s^\star\|_2,
\]
and
\[
\mathbb{E}\!\left[\langle v_{\mathrm{ggr}}, \hat g_s^\star\rangle\right]
=
\bigl(1+\lambda_u \mathbb{E}[\max(0,\beta)]\bigr)\|g_s^\star\|_2.
\]
Consequently, GGR never underperforms filtering in expectation, and it is strictly better whenever \(\lambda_u>0\) and \(\mathbb{P}(\beta>0)>0\). In contrast, naive utilization underperforms filtering whenever \(\mathbb{E}[\beta]<0\).
\end{proposition}

\begin{proof}
The filtering case is immediate:
\[
\langle v_{\mathrm{filter}}, \hat g_s^\star\rangle
=
\langle g_s^\star, \hat g_s^\star\rangle
=
\|g_s^\star\|_2.
\]

For naive utilization,
\[
\langle v_{\mathrm{naive}}, \hat g_s^\star\rangle
=
\langle g_s^\star + \lambda_u(\beta g_s^\star + \epsilon_u), \hat g_s^\star\rangle
=
(1+\lambda_u \beta)\|g_s^\star\|_2,
\]
since \(\langle \epsilon_u,\hat g_s^\star\rangle = 0\). Taking expectation gives the second identity.

For GGR, observe that
\[
\langle g_u, g_s^\star\rangle
=
\beta \|g_s^\star\|_2^2,
\]
because \(\epsilon_u\) is orthogonal to \(g_s^\star\). If \(\beta \ge 0\), then \(g_u\) is already feasible, so the projection leaves it unchanged. If \(\beta < 0\), the projection removes exactly the negative component along \(g_s^\star\), yielding
\[
\Pi_{\mathcal{H}_{\mathrm{safe}}(g_s^\star)}(g_u) = \epsilon_u.
\]
Hence
\[
\left\langle
\Pi_{\mathcal{H}_{\mathrm{safe}}(g_s^\star)}(g_u),
\hat g_s^\star
\right\rangle
=
\max(0,\beta)\|g_s^\star\|_2,
\]
and therefore
\[
\langle v_{\mathrm{ggr}}, \hat g_s^\star\rangle
=
\bigl(1+\lambda_u \max(0,\beta)\bigr)\|g_s^\star\|_2.
\]
Taking expectation yields the result.
\end{proof}

\subsection{Relation to symmetric gradient surgery}
\label{app:pcgrad}

GGR should be interpreted as an \emph{asymmetric anchor-preserving projection} rather than a symmetric task-balancing rule. In OSSL, the supervised gradient is label-verified and therefore acts as the protected anchor, whereas the auxiliary gradient can be corrupted by pseudo-label mistakes or open-set objectives. This asymmetry is exactly where GGR differs from standard PCGrad \cite{pcgrad}.

\begin{proposition}[Asymmetric anchoring versus symmetric PCGrad]
\label{prop:pcgrad}
Let \(g_s \neq 0\), \(g_u \neq 0\), and \(c=\langle g_s,g_u\rangle<0\). Let
\[
v_{\mathrm{ggr}}
=
g_s + \left(g_u - \frac{c}{\|g_s\|_2^2}g_s\right)
\]
be the VLR update, and let
\[
v_{\mathrm{pcgrad}}
=
\left(g_s - \frac{c}{\|g_u\|_2^2}g_u\right)
+
\left(g_u - \frac{c}{\|g_s\|_2^2}g_s\right)
\]
be the standard symmetric PCGrad aggregation for the pair \((g_s,g_u)\). Then
\[
\langle g_s, v_{\mathrm{ggr}}\rangle = \|g_s\|_2^2,
\]
whereas
\[
\langle g_s, v_{\mathrm{pcgrad}}\rangle
=
\|g_s\|_2^2 - \frac{c^2}{\|g_u\|_2^2}
=
\|g_s\|_2^2\bigl(1-\cos^2(g_s,g_u)\bigr).
\]
In particular, if \(g_s\) and \(g_u\) are nearly opposite, then the supervised progress under symmetric PCGrad approaches zero.
\end{proposition}

\begin{proof}
For VLR, Proposition~\ref{prop:halfspace} gives
\[
\tilde g_u = g_u - \frac{c}{\|g_s\|_2^2}g_s.
\]
Therefore,
\[
\langle g_s, v_{\mathrm{ggr}}\rangle
=
\langle g_s, g_s\rangle
+
\left\langle g_s, g_u - \frac{c}{\|g_s\|_2^2}g_s \right\rangle
=
\|g_s\|_2^2 + c - c
=
\|g_s\|_2^2.
\]

For symmetric PCGrad,
\[
\langle g_s, v_{\mathrm{pcgrad}}\rangle
=
\left\langle g_s, g_s - \frac{c}{\|g_u\|_2^2}g_u \right\rangle
+
\left\langle g_s, g_u - \frac{c}{\|g_s\|_2^2}g_s \right\rangle
=
\|g_s\|_2^2 - \frac{c^2}{\|g_u\|_2^2}.
\]
Using \(c=\langle g_s,g_u\rangle=\|g_s\|_2\|g_u\|_2\cos(g_s,g_u)\) yields
\[
\frac{c^2}{\|g_u\|_2^2}
=
\|g_s\|_2^2\cos^2(g_s,g_u),
\]
which proves the stated identity.
\end{proof}

If PCGrad is artificially constrained to rectify only the auxiliary gradient while keeping \(g_s\) fixed, it reduces exactly to the VLR update in Eq.~\eqref{eq:ggr_vector}. Therefore, the OSSL-specific design choice is not projection alone, but the asymmetric preservation of the label-verified supervised anchor.

We further compare GGR against generic plug-in optimization controls using the same base hyperparameters as their corresponding OSSL baselines. The goal is to test whether symmetric conflict handling or simple gradient suppression can replace the supervised-anchor design.

\begin{table}[t]
\centering
\footnotesize
\setlength{\tabcolsep}{3pt}
\caption{Control study comparing GGR with generic optimization controls. Each cell reports closed-set accuracy / open-set BA. PCGrad denotes the standard symmetric formulation, GradClip clips the auxiliary-gradient norm, and ConfDrop discards \(g_u\) when \(\langle g_s,g_u\rangle<0\).}
\label{tab:control_pcgrad}
\resizebox{\linewidth}{!}{
\begin{tabular}{lcccc}
\toprule
Method & C10-6-30 & C100-20-100 & C100-50-250 & ImageNet-1\% \\
\midrule
DAC & 89.96 / 74.13 & 55.00 / 49.82 & 56.82 / 54.75 & 86.50 / 80.42 \\
+ PCGrad & 89.28 / 74.00 & 56.40 / 50.15 & 55.82 / 53.89 & 86.25 / 75.97 \\
+ GradClip & 85.02 / 70.01 & 50.60 / 47.52 & 56.06 / 49.58 & 86.60 / 80.23 \\
+ ConfDrop & 90.42 / 74.86 & 53.70 / 50.35 & 55.08 / 53.12 & 85.85 / 76.98 \\
+ Ours & \textbf{90.85} / \textbf{76.03} & \textbf{58.20} / \textbf{54.31} & \textbf{58.58} / \textbf{56.60} & \textbf{87.20} / \textbf{82.71} \\
\midrule
IOMatch & 89.78 / 72.69 & 55.35 / 45.17 & 58.38 / 46.81 & 86.60 / 75.27 \\
+ PCGrad & 91.22 / 74.74 & 56.50 / 47.05 & 59.72 / 48.32 & 85.15 / 74.75 \\
+ GradClip & 92.82 / 73.75 & 52.25 / 41.71 & 56.36 / 44.30 & 85.00 / 74.97 \\
+ ConfDrop & 85.63 / 69.16 & 55.65 / 44.78 & 59.18 / 47.39 & 83.95 / 73.00 \\
+ Ours & \textbf{93.43} / \textbf{76.32} & \textbf{58.20} / \textbf{47.72} & \textbf{61.02} / \textbf{50.90} & \textbf{87.40} / \textbf{76.24} \\
\bottomrule
\end{tabular}
}
\end{table}

Table~\ref{tab:control_pcgrad} complements Proposition~\ref{prop:pcgrad} with direct empirical evidence. Symmetric PCGrad can occasionally help, but it is markedly less stable across DAC and IOMatch because it may also distort the supervised anchor. GradClip uniformly shrinks useful and harmful components together, while ConfDrop discards the entire auxiliary signal whenever a conflict is detected. GGR is consistently strongest because it removes only the destructive component of \(g_u\) while preserving both the supervised anchor and the non-conflicting auxiliary information.

\section{Implementation Details}
\label{app:implementation}

All methods are implemented and evaluated in the USB codebase to ensure identical data processing, augmentation pipelines, and evaluation protocols across baselines and GGR-enhanced variants. For CIFAR-10/100 \cite{cifar} we use WideResNet-28-2 \cite{wideresnet}, and for ImageNet-30 \cite{imagenet} we use ResNet-18 \cite{resnet}. Unless otherwise specified, ``GGR'' refers to the default vector-level rectifier (VLR) in Eq.~\eqref{eq:ggr_vector}. Following the main paper, gradient rectification is applied to the encoder backbone by default, while task-specific heads are updated with the original gradients. This design isolates representation-level conflict mitigation and preserves the method-specific behavior of the original OSSL heads.

When GGR is attached to a base OSSL method, all base-method-specific hyperparameters are kept identical to those of the corresponding baseline. The only intervention is the post-hoc rectification of the auxiliary gradient before the optimizer step. Checkpoints are selected according to the best closed-set validation accuracy, and the selected checkpoint is then used to report both closed-set accuracy and open-set BA. All reported results are averaged over three random seeds.

\section{Additional Diagnostics and Sensitivity Analyses}
\label{app:diagnostics}

\subsection{Conflict metrics}

To quantify destructive interference, we define the raw conflict indicator at iteration \(t\) as
\[
c_t^{\mathrm{raw}}
=
\mathbf{1}\!\left\{
\langle g_s^{(t)}, g_u^{(t)}\rangle < 0
\right\}.
\]
Its moving-average rate over a window of size \(W\) is
\[
\mathrm{CR}_t^{\mathrm{raw}}
=
\frac{1}{W}
\sum_{i=t-W+1}^{t}
c_i^{\mathrm{raw}}.
\]

For the actually applied auxiliary update \(g_{\mathrm{app}}^{(t)}\)---equal to \(\lambda_u g_u^{(t)}\) for the baseline and \(\lambda_u \tilde g_u^{(t)}\) for GGR---we additionally report
\[
c_t^{\mathrm{app}}
=
\mathbf{1}\!\left\{
\langle g_s^{(t)}, g_{\mathrm{app}}^{(t)}\rangle < 0
\right\},
\qquad
\mathrm{CR}_t^{\mathrm{app}}
=
\frac{1}{W}
\sum_{i=t-W+1}^{t}
c_i^{\mathrm{app}},
\]
and the cumulative applied conflict regret
\[
\mathcal{R}_T^{\mathrm{app}}
=
\sum_{t=1}^{T}
\max\!\Bigl(
0,\,
-\langle g_s^{(t)}, g_{\mathrm{app}}^{(t)}\rangle
\Bigr).
\]
For VLR, Corollary~\ref{cor:zero_regret} implies \(\mathcal{R}_T^{\mathrm{app}} = 0\) in exact arithmetic. In practice, tiny numerical deviations may appear under mixed-precision training or gradient accumulation, but the value remains negligible compared with the baseline.

For all diagnostics in Figure~\ref{fig:conflict_regret}, gradients are measured on the same surgery scope used by the optimizer update. The moving-average window is set to 1000 iterations.

\subsection{Computational complexity}

Let \(D\) denote the number of parameters inside the selected surgery scope and let \(d\) denote the subspace dimension. Once the supervised and auxiliary gradients are available, the VLR projection step requires one inner product and one vector subtraction, leading to \(O(D)\) time and \(O(1)\) extra memory. For a fixed basis \(U \in \mathbb{R}^{D\times d}\), the OSR and CSR projection steps each require two matrix-vector multiplications, leading to \(O(Dd)\) projection time and \(O(Dd)\) memory to store the basis. These bounds describe only the algebraic rectification cost. The full training overhead also includes computing the supervised and auxiliary gradients separately, assembling the rectified update for the optimizer step, and, for subspace variants, maintaining and re-orthonormalizing the running supervised subspace. This explains why VLR is adopted as the default choice in the main tables, whereas subspace variants trade extra computation for a more stable anchor under noisy supervised gradients.

\end{document}